\documentclass[twoside]{article}
\pdfoutput=1
\usepackage[accepted]{aistats2024}
\usepackage{algos}

\renewcommand{\S}{\mathcal{S}}
\newcommand{\A}{\mathcal{A}}
\renewcommand{\H}{\mathcal{H}}
\newcommand{\mS}{\overline{\mathcal{S}}}
\newcommand{\mA}{\overline{\mathcal{A}}}
\newcommand{\mP}{\overline{P}}
\newcommand{\mr}{\overline{r}}
\newcommand{\mR}{\overline{R}}

\newcommand{\mM}{\overline{M}}

\newcommand{\ms}{\overline{s}}
\newcommand{\ma}{\overline{a}}
\newcommand{\mV}{\overline{V}}
\newcommand{\D}{D}
\renewcommand{\c}{\overline{c}}
\renewcommand{\E}{\mathbb{E}}
\renewcommand{\P}{\mathbb{P}}
\newcommand{\cmax}{c^{\mathrm{max}}}
\newcommand{\cmin}{c^{\mathrm{min}}}

\newcommand{\hS}{\hat{\mathcal{S}}}
\newcommand{\hA}{\hat{\mathcal{A}}}
\newcommand{\hP}{\hat{P}}

\newcommand{\hR}{\hat{R}}

\newcommand{\hM}{\hat{M}}
\newcommand{\hs}{\hat{s}}

\newcommand{\hV}{\hat{V}}
\usepackage[round]{natbib}

\begin{document}

% If your paper is accepted and the title of your paper is very long,
% the style will print as headings an error message. Use the following
% command to supply a shorter title of your paper so that it can be
% used as headings.
%
%\runningtitle{I use this title instead because the last one was very long}

% If your paper is accepted and the number of authors is large, the
% style will print as headings an error message. Use the following
% command to supply a shorter version of the authors names so that
% they can be used as headings (for example, use only the surnames)
%
%\runningauthor{Surname 1, Surname 2, Surname 3, ...., Surname n}

\twocolumn[

\aistatstitle{Anytime-Constrained Reinforcement Learning}

\aistatsauthor{ Jeremy McMahan \And Xiaojin Zhu }

\aistatsaddress{ University of Wisconsin-Madison } ]

\begin{abstract}
  We introduce and study constrained Markov Decision Processes (cMDPs) with anytime constraints. An anytime constraint requires the agent to never violate its budget at any point in time, almost surely. Although Markovian policies are no longer sufficient, we show that there exist optimal deterministic policies augmented with cumulative costs. In fact, we present a fixed-parameter tractable reduction from anytime-constrained cMDPs to unconstrained MDPs. Our reduction yields planning and learning algorithms that are time and sample-efficient for tabular cMDPs so long as the precision of the costs is logarithmic in the size of the cMDP. However, we also show that computing non-trivial approximately optimal policies is NP-hard in general. To circumvent this bottleneck, we design provable approximation algorithms that efficiently compute or learn an arbitrarily accurate approximately feasible policy with optimal value so long as the maximum supported cost is bounded by a polynomial in the cMDP or the absolute budget. Given our hardness results, our approximation guarantees are the best possible under worst-case analysis.
\end{abstract}

\section{INTRODUCTION}\label{sec: intro}
Suppose $M$ is a constrained Markov Decision Process (cMDP). An \emph{anytime constraint} requires that cost accumulated by the agent's policy $\pi$ is within the budget at any time, almost surely: $\P^{\pi}_M\brac{\forall k \in [H], \; \sum_{t = 1}^k c_t \leq B} = 1$. If $\Pi_M$ denotes the set of policies that respect the anytime constraints, then a solution to the anytime-constrained cMDP is a policy $\pi^* \in \argmax_{\pi \in \Pi_M} V^{\pi}_M$. For example, consider planning a minimum-time route for an autonomous vehicle to travel from one city to another. Besides time, there are other important considerations including~\citep{DrivingVirtualReal} (1) the route does not exhaust the vehicle's fuel, and (2) the route is safe. We can model (1) by defining the fuel consumed traveling a road to be its cost and the tank capacity to be the budget.
Refueling stations are captured using negative costs. We can model many safety considerations (2) similarly.

Since nearly every modern system is constrained in some way, one key step to modeling more realistic domains with MDPs is allowing constraints. To address this, a rich literature of constrained reinforcement learning (CRL) has been developed, almost exclusively focusing on expectation constraints~\citep{cMDP-book, cMDP-CPO, cMDP-Actor-Critic} or high-probability (chance) constraints~\citep{CCMDP-OG, CCMDP-DPwSpaceExploration, CCMDP-Percentile-Risk}.
However, in many applications, especially where safety is concerned or resources are consumed, anytime constraints are more natural. 
An agent would not be reassured by the overall expected or probable safety of a policy when it is faced with a reality or intermediate time where it is harmed. 
For instance, a car cannot run on expected or future gas, so must satisfy the $B$ budget at any time along the route.
Similarly, in goal-directed RL~\citep{RSMDP-Revisit, RSMDP-Budget, OG-SSP}, the goal must be achieved; maximizing reward is only a secondary concern.
These issues are especially crucial in medical applications~\citep{MedSurvery, MedRisk, MedScheduling}, disaster relief scenarios~\citep{DisasterDigitalTwin, DisasterUAV, DisasterFlooding}, and resource management~\citep{ResourceGeneral, ResourceNetworkSlicing, ResourceUAV, ResourceMatching}. 

Anytime constraints are natural, but introduce a plethora of new challenges. Traditional Markovian and history-dependent policies are rarely feasible and can be arbitrarily suboptimal; the cost history must also be considered.
Naively using backward induction to compute an optimal cost-history-dependent policy is possible in principle for tabular cost distributions, but the time needed to compute the policy and the memory needed to store the policy would be super-exponential. Since the optimal solution value is a discontinuous function of the costs, using standard CRL approaches like linear programming is also impossible. In fact, not only is computing an optimal policy NP-hard but computing any policy whose value is approximately optimal is also NP-hard when at least two constraints are present.

Known works fail to solve anytime-constrained cMDPs. Expectation-constrained approaches~\citep{cMDP-book, cMDP-ZeroDualityGap, cMDP-SafeExploration, cMDP-PAC-H, cMDP-Regret-H} and chance-constrained approaches~\citep{CCMDP-DPwSpaceExploration, CCMDP-GoalComplexity, CCMDP-Percentile-Risk, CCMDP-RL-Batch} yield policies that arbitrarily violate an anytime constraint. This observation extends to nearly every known setting: Knapsack constraints~\citep{Knap-Brantley, Knap-PreBrantley, Knap-RLwK}, risk constraints~\citep{RCMDP-OG, CCMDP-Percentile-Risk}, risk sensitivity~\citep{RSMDP-OG, RSMDP-Revisit, RSMDP-GoalProbPlanning, RSMDP-Budget}, quantile constraints~\citep{QuantileConstrained, PreQuantile}, and instantaneous constraints~\citep{InstantaneousSafeRL,  PreInstantaneous1, PreInstantaneous2}. If we were instead to use these models with a smaller budget to ensure feasibility, the resultant policy could be arbitrarily suboptimal if any policy would be produced at all. Dangerous-state~\citep{SafeStatePAC, Safe-RL-Imagining} and almost-sure~\citep{AlmostSure} constraints can be seen as a special case of our model with binary costs. However, their techniques do not generalize to our more complex setting.
Moreover, absent expectation constraints, none of these approaches are known to admit polynomial time planning or learning algorithms.

\paragraph{Our Contributions.} 
We present the first formal study of anytime-constrained cMDPs. Although traditional policies do not suffice, we show that deterministic augmented policies are always optimal. In fact, an optimal policy can be computed by solving an unconstrained, augmented MDP using any standard RL planning or learning algorithm. Using the intuition of safe exploration and an atypical forward induction, we derive an augmented state space rich enough to capture optimal policies without being prohibitively large. To understand the resultant augmented policies, we design new machinery requiring a combination of backward and forward induction to argue about optimality and feasibility. Overall, we show our reduction to standard RL is \emph{fixed-parameter tractable} (FPT)~\citep{FPT} in the cost precision when the cMDP and cost distribution are tabular. In particular, as long as the cost precision is logarithmic in the size of the cMDP, our planning (learning) algorithms are polynomial time (sample complexity), and the produced optimal policy can be stored with polynomial space.

Since we show computing any non-trivial approximately-optimal policy is NP-hard, we turn to approximate feasibility for the general case. 
For any $\epsilon > 0$, we consider additive and relative approximate policies that accumulate cost at most $B+\epsilon$ and $B(1+\epsilon)$ anytime, respectively. 
\footnote{Critically, we can also ensure strict budget $B$ feasibility but with a weaker value guarantee. See section~\ref{sec:approx}.}
Rather than consider every cumulative cost induced by safe exploration, our approximation scheme inductively accumulates and projects the costs onto a smaller space. Following the principle of optimism, the approximate cost is constructed to be an underestimate to guarantee optimal value. Our approach yields planning (learning) algorithms that produce optimal value, and approximately feasible policies in polynomial time (sample complexity) for any, possibly non-tabular, cost distribution whose maximum supported cost is bounded by a polynomial in the cMDP or by the absolute budget. Given our hardness results, this is the best possible approximation guarantee one could hope for under worst-case analysis. 
We also extend our methods to handle different budgets per time, general almost-sure constraints, and infinite discounting.

\subsection{Related Work.} 
\paragraph{Knapsack Constraints.} The knapsack-constrained frameworks \citep{Knap-Brantley, Knap-PreBrantley, Knap-RLwK} were developed to capture constraints on the learning process similar to bandits with knapsacks~\citep{Knap-BwK}. \citet{Knap-Brantley} and \citet{Knap-PreBrantley} both constrain the total cost violation that can be produced during training. On the other hand, \citet{Knap-RLwK} introduces the RLwK framework that constrains the total cost used per episode. In RLwK, each episode terminates when the agent violates the budget for that episode. In all of these models, the environment is given the ability to terminate the process early; the final policy produced after learning need not satisfy any kind of constraint. In fact, the agent still keeps the reward it accumulated before violation, the agent is incentivized to choose unsafe actions in order to maximize its reward. Thus, such methods produce infeasible policies for our anytime constraints regardless of the budget they are given. 

\paragraph{Almost Sure Constraints.} Performing RL while avoiding dangerous states~\citep{SafeStatePAC, Safe-RL-Imagining, SafeSurvey} can be seen as a special case of both anytime and expectation constraints with binary costs and budget $0$. However, these works require non-trivial assumptions, and being a special case of expectation constraints implies their techniques cannot solve our general setting. Similarly, \citet{AlmostSure} introduced almost sure constraints with binary costs, which can be seen as a special case of anytime constraints. However, they focus on computing minimal budgets, which need not lead to efficient solutions in general since the problem is NP-hard even with a budget of $1$. Lastly, the infinite time-average case with almost sure constraints has been thoroughly studied~\citep{TimeAverage}. Since we focus on finite-horizon and discounted settings, our policies would always have a time-average cost of $0$ and so those methods cannot produce policies that are feasible for anytime constraints. 

\section{ANYTIME CONSTRAINTS}\label{sec: constraints}

\paragraph{Constrained Markov Decision Processes.} A (tabular, finite-horizon) Constrained Markov Decision Process is a tuple $M = (\S, \A, P, R, H, s_0, C, B)$, where (i) $\S$ is a finite set of states, (ii) $\A$ is a finite set of actions, (iii) $P_h(s,a) \in \Delta(S)$ is the transition distribution, (iv) $R_h(s,a) \in \Delta(\Real)$ is the reward distribution, (v) $H \in \Nat$ is the finite time horizon, (vi) $s_0 \in \S$ is the initial state, (vii) $C_h(s,a) \in \Delta(\Real^d)$ is the cost distribution, and (viii) $B \in \Real^d$ is the budget vector. Here, $d \in \Nat$ denotes the number of constraints. We overload notation by letting $C_h(s,a)$ denote both the cost distribution and its support. We also let $r_h(s,a) = \E[R_h(s,a)]$ denote the expected reward. Lastly, we let $S := |\S|$, $A := |\A|$, $[H] := \set{1, \ldots, H}$, and $|M|$ be the description size of the cMDP.

\paragraph{Interaction Protocol.} A complete history with costs takes the form $\tau = (s_1, a_1, c_1, \allowbreak \ldots, s_H, a_H, \allowbreak c_H, \allowbreak s_{H+1})$, where $s_h \in \S$ denotes $M$'s state at time $h$, $a_h \in \A$ denotes the agent's chosen action at time $h$, and $c_h \in C_h(s_h,a_h)$ denotes the cost incurred at time $h$. We let $\c_h := \sum_{t = 1}^{h-1} c_t$ denote the cumulative cost up to (but not including) time $h$. Also, we denote by $\tau_h = (s_1, a_1, c_1, \ldots, s_h)$ the partial history up to time $h$ and denote by $\H_h$ the set of partial histories up to time $h$. The agent interacts with $M$ using a policy $\pi = (\pi_h)_{h = 1}^H$, where $\pi_h : \H_h \to \Delta(\A)$ specifies how the agent chooses actions at time $h$ given a partial history. 

The agent starts at state $s_0$ with partial history $\tau_1 = (s_0)$. For any $h \in [H]$, the agent chooses an action $a_h \sim \pi_h(\tau_h)$. Afterward, the agent receives reward $r_h \sim R_h(s_h,a_h)$ and cost $c_h \sim C_h(s_h,a_h)$. Then, $M$ transitions to state $s_{h+1} \sim P_h(s_h,a_h)$ and the history is updated to $\tau_{h+1} = (\tau_h, a_h, c_h, s_{h+1})$. This process is repeated for $H$ steps total; the interaction ends once $s_{H+1}$ is reached. 

\paragraph{Objective.} The agent's goal is to compute a $\pi^*$ that is a solution to the following optimization problem:
\begin{equation}\tag{ANY}\label{equ: objective}
    \begin{split}
        \max_{\pi} \; &\E^{\pi}_M\brac{\sum_{h = 1}^H r_h(s_h,a_h)} \\ 
        \text{s.t.} \; &\P^{\pi}_M\brac{\forall t \in [H], \; \sum_{h = 1}^t c_h \leq B} = 1.
    \end{split}
\end{equation}
Here, $\P^{\pi}_M$ denotes the probability law over histories induced from the interaction of $\pi$ with $M$, and $\E^{\pi}_M$ denotes the expectation with respect to this law.  
We let $V^{\pi} := \E^{\pi}_M\brac{\sum_{t = 1}^H r_t(s_t,a_t)}$ denote the value of a policy $\pi$, $\Pi_M := \set{\pi \mid \P^{\pi}_M\brac{\forall k \in [H], \; \sum_{t = 1}^k c_t \leq B} = 1}$ denote the set of feasible policies, and $V^* := \max_{\pi \in \Pi_M} V^{\pi}$ denote the optimal solution value.
\footnote{By using a negative budget, we capture the covering constraints that commonly appear in goal-directed problems. We consider the other variations of the problem in the Appendix.}. 
If there are no feasible policies, $V^*:=-\infty$ by convention.

\paragraph{Optimal Solutions.} It is well-known that expectation-constrained cMDPs always admit a randomized Markovian policy~\citep{cMDP-book}. However, under anytime constraints, 
feasible policies that do not remember the cumulative cost can be arbitrarily suboptimal. The intuition is that without knowing the cumulative cost, a policy must either play it too safe and suffer small value or risk an action that violates the constraint. 
\begin{proposition}\label{prop: Markovian-suboptimality}
    Any class of policies that excludes the full cost history is suboptimal for anytime-constrained cMDPs. In particular, Markovian policies can be arbitrarily suboptimal even for cMDPs with $S = 1$ and $A = H = 2$. 
\end{proposition}

\begin{corollary}\label{cor: cMDP-failure}
    (Approximately) optimal policies for cMDPs with expectation constraints, chance constraints, or their variants can arbitrarily violate an anytime constraint. Furthermore, (approximately) optimal policies for a cMDP defined by a smaller budget to achieve feasibility can be arbitrarily suboptimal. 
\end{corollary}

Although using past frameworks out of the box does not suffice, one might be tempted to use standard cMDP techniques, such as linear programming, to solve anytime-constrained problems. However, continuous optimization techniques fail since the optimal anytime-constrained value is discontinuous in the costs and budgets. Even a slight change to the cost or budget can lead to a dramatically smaller solution value.

\begin{proposition}\label{prop: cost-sensitivity}
    $V^*$ is a continuous function of the rewards, but a discontinuous function of the costs and budgets.
\end{proposition}

\paragraph{Intractability.} In fact, solving anytime-constrained cMDPs is fundamentally harder than expectation-constrained cMDPs; solving \eqref{equ: objective} is NP-hard. The intuition is that anytime constraints capture the knapsack problem. With a single state, we can let the reward at time $i$ be item $i$'s value and the cost at time $i$ be item $i$'s weight. Any deterministic policy corresponds to choosing certain items to add to the knapsack, and a feasible policy ensures the set of items fit in the knapsack. Thus, an optimal deterministic policy for the cMDP corresponds to an optimal knapsack solution.

On the other hand, a randomized policy does not necessarily yield a solution to the knapsack problem. However, we can show that any randomized policy can be derandomized into a deterministic policy with the same cost and at least the same value. The derandomization can be performed by inductively choosing any supported action that leads to the largest value. This is a significant advantage over expectation and chance constraints which typically require stochastic policies.

\begin{lemma}[Derandomization]\label{lem: derandomization}
    For any randomized policy $\bar \pi$, there exists a deterministic policy $\pi$ whose cumulative cost is at most $\bar \pi$'s anytime and whose value satisfies $V^{\pi} \geq V^{\bar \pi}$. 
\end{lemma}

Since the existence of a randomized solution implies the existence of a deterministic solution with the same value via \cref{lem: derandomization}, the existence of a high-value policy for an anytime-constrained cMDP corresponds to the existence of a high-value knapsack solution. Thus, anytime constraints can capture the knapsack problem.
Our problem remains hard even if we restrict to the very specialized class of deterministic, non-adaptive (state-agnostic) policies, which are mappings from time steps to actions: $\pi: [H] \to \A$.  

\begin{theorem}[Hardness]\label{thm: np-hardness}
    Solving \eqref{equ: objective} is NP-complete even when $S = 1$, $A = 2$, and both the costs and rewards are deterministic, non-negative integers. This remains true even if restricted to the class of non-adaptive policies. Hardness also holds for stationary cMDPs so long as $S \geq H$.
\end{theorem}

Given the hardness results in \cref{thm: np-hardness}, it is natural to turn to approximation algorithms to find policies efficiently. The most natural approach would be to settle for a feasible, although, approximately-optimal policy. Unfortunately, even with only $d = 2$ constraints, it is intractable to compute a feasible policy with any non-trivial approximation factor. This also means designing an algorithm whose complexity is polynomial in $d$ is likely impossible.

\begin{theorem}[Hardness of Approximation]\label{thm: approx-hard}
    For $d \geq 2$, computing a feasible solution to \eqref{equ: objective} is NP-hard.
    Furthermore, for any $\epsilon > 0$, it is NP-hard to compute a feasible policy $\pi$ satisfying either $V^{\pi} \geq V^* - \epsilon$ or $V^{\pi} \geq V^*(1-\epsilon)$.
\end{theorem}

\begin{remark}
    Note, \cref{thm: approx-hard} does not only rule out the existence of fully-polynomial-time approximation schemes (FPTAS). Since $\epsilon > 0$ is arbitrary, it rules out any non-trivial approximation similar to the (non-metric) Traveling Salesman Problem.
\end{remark}

\section{FPT REDUCTION}\label{sec: reduction}

Despite our strong hardness results, \cref{thm: np-hardness} and \cref{thm: approx-hard}, we show for a large class of cMDPs, \eqref{equ: objective} can be solved efficiently. The key is to augment the state space of the system to capture the constraint consideration. In this section, we assume the cost distributions have finite support; we generalize to broader classes of distributions in \cref{sec: approximation}.
\begin{assumption}\label{assump: discrete}
    $n := \sup_{h,s,a} |C_h(s,a)| < \infty$.
\end{assumption}

\cref{prop: Markovian-suboptimality} illustrates that a key issue with standard policies is that they cannot adapt to the costs seen so far. This forces the policies to be overly conservative or to risk violating the budget. At the same time, cost-history-dependent policies are undesirable as they are computationally expensive to construct and store in memory. 

Instead, we claim the agent can exploit a sufficient statistic of the cost sequence: the \emph{cumulative cost}. By incorporating cumulative costs carefully, the agent can simulate an unconstrained MDP, $\mM$, whose optimal policies are solutions to \eqref{equ: objective}.
The main challenge is defining the augmented states, $\mS_h$. 

\paragraph{Augmented States.} We could simply define $\mS_h$ to be $\S \times \Real^d$, but this would result in an infinite state MDP with a discontinuous reward function, which cannot easily be solved. The ideal choice would be $\mathcal{F}_h := \{(s,\c) \in \S \times \Real^d \mid \exists \pi \in \Pi_M, \; \P^{\pi}_M\brac{s_h = s, \c_h = \c} > 0\}$, which is the minimum set containing all (state, cost)-pairs induced by feasible policies. However, $\mathcal{F}_h$ is difficult to characterize. 

Instead, we consider a relaxation stemming from the idea of \emph{safe exploration}. Namely, we look at the set of all (state, cost)-pairs that the agent could induce if it repeatedly interacted with $M$ and only took actions that would not violate the constraint given the current history. 
This set can be constructed inductively. First, the agent starts with $(s_0,0)$ because it has yet to incur any costs. Then, if at time $h$, the agent has safely arrived at the pair $(s,\c)$, the agent can now safely choose any action $a$ for which $\Pr_{c \sim C_h(s,a)}\brac{\c + c \leq B} = 1$. 
\begin{definition}[Augmented States]\label{def: states}
$\mS_1 := \set{(s_0, 0)}$, and for any $h \geq 1$,
\begin{equation*}\label{equ: ms}
\begin{split}
    \mS_{h+1} := \Big\{(s',\c') \mid \exists (s,\c) \in \mS_h, a \in \A, c' \in C_h(s,a), \\
    \c' = \c + c', \; \Pr_{c \sim C_h(s,a)}[\c + c \leq B] = 1, \; \\ P_h(s' \mid s,a) > 0 \Big\}.
\end{split}
\end{equation*}
\end{definition}

Unlike the backward induction approaches commonly used in MDP theory, observe that $\mS$ is constructed using \emph{forward induction}. This feature is critical to computing a small, finite augmented-state space. We also point out that $\mS_h$ is a relaxation of $\mathcal{F}_h$ since actions chosen based on past costs without considering the future may not result in a fully feasible path. Nevertheless, the relaxation is not too weak; whenever $0$ cost actions are always available, $\mS_h$ exactly matches $\mathcal{F}_h$.

\begin{lemma}\label{lem: feasibility}
    $\forall h \in [H+1]$, $\mS_h \supseteq \mathcal{F}_h$ and $|\mS_h| < \infty$. Furthermore, equality holds if $\forall h,s, \exists a$ for which $C_h(s,a) = \{0\}$. 
\end{lemma}

If the agent records its cumulative costs and always takes safe actions, the interaction evolves according to the following unconstrained MDP.

\begin{definition}[Augmented MDP]\label{def: augmentation}
The \emph{augmented MPD} $\mM := (\mS, \mA, \mP, \mR, \allowbreak H, \ms_0)$ where,

\begin{itemize}
    \item $\mS_h$ is defined in \cref{def: states}.
    \item $\mA_h(s, \c) := \set{a \in \A \mid \Pr_{c \sim C_h(s,a)}\brac{\c + c \leq B} = 1}$. 
    \item $\mP_h((s',\c+c) \mid (s,\c), a) := P_h(s' \mid s, a) C_h(c \mid s,a)$. 
    \item $\mR_h((s,\c), a) := R_h(s,a)$.
    \item $\ms_0 := (s_0, 0)$.
\end{itemize}
\end{definition}

\begin{theorem}[Optimality]\label{thm: optimality}
    \cref{alg: reduction} solves \eqref{equ: objective} and can be implemented to run in finite time.
\end{theorem}

\begin{algorithm}[t]
\caption{Reduction to Unconstrained RL}\label{alg: reduction}
\begin{algorithmic}[1]
\Require{cMDP $M$}
\State $\mM \gets $ \cref{def: augmentation}$(M)$
\State $\pi, \mV^* \gets \text{Solve}(\mM)$
\If{ $\mV^* = -\infty$}
    \State \Return ``Infeasible"
\Else 
    \State \Return $\pi$
\EndIf
\end{algorithmic}
\end{algorithm}

We see from \cref{thm: optimality} that an anytime-constrained cMDP can be solved using \cref{alg: reduction}. If $M$ is known, the agent can directly construct $\mM$ using \cref{def: augmentation} and then solve $\mM$ using any RL planning algorithm. If $M$ is unknown, the agent can still solve $\mM$ by replacing the call to $\pi_h(s,\c)$ in \cref{alg: protocol} by a call to any RL learning algorithm.

\begin{corollary}[Reduction]\label{cor: reduction}
    An optimal policy for an anytime-constrained cMDP can be computed from \cref{alg: reduction} paired with any RL planning or learning algorithm. Thus, anytime-constrained RL reduces to standard RL.
\end{corollary}

\paragraph{Augmented Policies.} Observe that any Markovian policy $\pi$ for $\mM$ is a \emph{augmented policy} that maps (state, cost)-pairs to actions. This policy can be translated into a full history policy or can be used directly through the new interaction protocol described in \cref{alg: protocol}. By recording the cumulative cost, the agent effectively simulates the $\pi-\mM$ interaction through the $\pi-M$ interaction.

\begin{algorithm}[t]
\caption{Augmented Interaction Protocol}\label{alg: protocol}
\textbf{Input:} augmented policy $\pi$
\begin{algorithmic}[1]
\State $\ms_1 = (s_0, 0)$ and $\c_1 = 0$.
\For{$h = 1$ to $H$}
    \State $a_h = \pi_h(\ms_h)$.
    \State $c_h \sim C_h(s_h,a_h)$ and $s_{h+1} \sim P_h(s_h,a_h)$.
    \State $\c_{h+1} = \c_h + c_h$.
    \State $\ms_{h+1} = (s_{h+1}, \c_{h+1})$.
\EndFor
\end{algorithmic}
\end{algorithm}

\paragraph{Analysis.} To understand augmented policies, we need new machinery than typical MDP theory. Since traditional policies are insufficient for anytime constraints, we need to directly compare against cost-history-dependent policies. However, we cannot consider arbitrary histories, since an infeasible history could allow higher value. Rather, we focus on histories that are induced by safe exploration: 
\begin{equation*}
    \begin{split}
        W_h(s, \c) := \Big\{\tau_h \in \H_h \mid \exists \pi, \; \P^{\pi}_{\tau_h}\brac{s_h = s, \c_h = \c} = 1, \\
        \P^{\pi}_{\tau_k}\brac{\c_{k+1} \leq B} = 1 \quad \forall k \in [h-1]\Big\}.
    \end{split}
\end{equation*}
Here, $\P^{\pi}_{\tau_h}[\cdot] := \P^{\pi}_{M}[\cdot \mid \tau_h]$ and $\E^{\pi}_{\tau_h}[\cdot] := \E^{\pi}_{M}[\cdot \mid \tau_h]$ denote the conditional probability and expectation given partial history $\tau_h$. The condition $\P^{\pi}_{\tau_k}[\c_{k+1} \leq B] = 1$ enforces that any action taken along the trajectory never could have violated the budget. 

We must also restrict to policies that are feasible given such a history: \(\Pi_M(\tau_h) := \{\pi \mid \P^{\pi}_{\tau_h}\brac{\forall k \in [H], \; \c_{k+1} \leq B} = 1\}.\)
Note that generally $\Pi_M(\tau_h) \supset \Pi_M$ so some $\pi \in \Pi_M(\tau_h)$ need not be feasible, but importantly $\Pi_M(s_0) = \Pi_M$ contains only feasible policies. We define,
\begin{equation*}
    V^*_h(\tau_h) := \max_{\pi \in \Pi_M(\tau_h)} V^{\pi}_{h}(\tau_h) \; ; \mV^*_h(s,\c) := \max_{\pi} \mV^{\pi}_{h}(s,\c),
\end{equation*}
to be the optimal feasible value conditioned on $\tau_h$, and the optimal value for $\mM$ from time $h$ onward starting from $(s,\c)$, respectively. We show that optimal feasible solutions satisfy the \emph{augmented bellman-optimality equations}.

\begin{lemma}\label{lem: value}
    For any $h \in [H+1]$, $(s,\c) \in \mS_h$, and $\tau_h \in W_h(s,\c)$, $\mV^*_h(s,\c) = V_h^*(\tau_h)$.
\end{lemma}
The proof is more complex than the traditional bellman-optimality equations. It requires (1) backward induction to argue that the value is maximal under any safe partial history and (2) forward induction to argue the costs accrued respect the anytime constraints. It then follows that solutions to $\mM$ are solutions to \eqref{equ: objective}.

\subsection{Complexity Analysis} 
To analyze the efficiency of our reduction, we define a combinatorial measure of a cMDP's complexity. 

\begin{definition}[Cost Diversity]\label{def: diversity}
The \emph{cost diversity}, $\D_M$, is the total number of distinct cumulative costs the agent could face at any time:
\begin{equation*}
    \D_M := \max_{h \in [H+1]}\abs{\set{c \mid \exists s, (s,c) \in \mS_h}}.
\end{equation*}
When clear from context, we refer to $\D_M$ as $\D$.
\end{definition}
We call $\D$ the diversity as it measures the largest cost population that exists in any generation $h$. The diversity naturally captures the complexity of an instance since the agent would likely encounter at least this many cumulative costs when computing or learning an optimal policy using any safe approach. 

In particular, we can bound the complexity of the planning and learning algorithms produced from our reduction in terms of the diversity. For concreteness, we pair \cref{alg: reduction} with backward induction~\citep{cMDP-book} to produce a planning algorithm and with BPI-UCBVI~\citep{MDP-PAC} to produce a learning algorithm.

\begin{proposition}[Complexity]\label{prop: complexity}
    Using \cref{alg: reduction}, an optimal policy for an anytime-constrained cMDP can be computed in $O\paren{HS^2A n \allowbreak \D}$ time and learned with $\tilde O\paren{H^3 SA \D \log(\frac{1}{\delta})/\gamma^2}$ sample complexity. Furthermore, the amount of space needed to store the policy is $O\paren{HS\D}$.
\end{proposition}

In the worst case, $\D$ could be exponentially large in the time horizon. However, for many cMDPs, $\D$ is small. One key factor in controlling the diversity is the precision needed to represent the supported costs in memory.

\begin{lemma}[Precision]\label{lem: precision}
    If the cost precision is at most $k$, then $\D \leq H^d 2^{(k+1)d}$. 
\end{lemma}

We immediately see that when the costs have precision $k$, all of our algorithms have complexity polynomial in the size of $M$ and exponential in $k$ and $d$. By definition, this means our algorithms are fixed-parameter tractable in $k$ so long as $d$ is held constant. Moreover, we see as long as the costs can be represented with logarithmic precision, our algorithms have polynomial complexity.

\begin{theorem}[Fixed-Parameter Tractability]\label{thm: fpt}
    For constant $d$, if $k = O(\log(|M|))$, planning (learning) for anytime-constrained cMDPs can be performed in polynomial time (sample complexity) using \cref{alg: reduction}, and the computed policy can be stored with polynomial space. 
\end{theorem}

\begin{remark}
    Many cost functions can be represented with small precision. In practice, all modern computers use fixed precision numbers. So, in any real system, our algorithms have polynomial complexity. Although technically efficient, our methods can be prohibitively expensive when $k$ or $d$ is large. However, this complexity seems unavoidable since computing exact solutions to \eqref{equ: objective} is NP-hard in general.
\end{remark}

\section{APPROXIMATION ALGORITHMS}\label{sec: approximation}

Since \cref{thm: approx-hard} rules out the possibility of traditional value-approximation algorithms due to the hardness of finding feasible policies, we relax the requirement of feasibility. We show that even with a slight relaxation of the constraint, solutions with optimal value can be found efficiently. 
Conversely, we can satisfy the constraint but with a weaker guarantee on value.
Our approximate-feasibility methods can even handle infinite support distributions so long as they are bounded above. 

\begin{assumption}\label{assump: bounded}
    $\cmax := \sup_{h,s,a} \sup C_h(s,a) < \infty$.
\end{assumption}

If $H\cmax \leq B$, then every policy is feasible, which just leads to a standard unconstrained problem. A similar phenomenon happens if $\cmax \leq 0$. Thus, we assume WLOG that $H\cmax > B$ and $\cmax > 0$. 

\begin{definition}[Approximate Feasibility]\label{def: approx-feasibility}
    For any $\epsilon > 0$, a policy $\pi$ is $\epsilon$-additive feasible if,
    \begin{equation}
        \P^{\pi}_M\brac{\forall t \in [H], \; \sum_{h = 1}^t c_h \leq B + \epsilon} = 1,
    \end{equation}
    and $\epsilon$-relative feasible if,
    \begin{equation}
        \P^{\pi}_M\brac{\forall t \in [H], \; \sum_{h = 1}^t c_h \leq B(1 + \epsilon \sigma_B)} = 1,
    \end{equation}
    where $\sigma_B$ is the sign of $B$\footnote{When the costs and budgets are negative, negating the constraint yields $\sum_{t = 1}^H c_t \geq \abs{B}(1-\epsilon)$, which is the traditional notion of relative approximation for covering objectives.}.
\end{definition}

\paragraph{Approximation.} The key to reducing the complexity of our reduction is lowering the cost diversity. Rather than consider every cost that can be accumulated from safe exploration, the agent can consider a smaller set of approximate cumulative costs. Specifically, for any cumulative cost $\c_h$ and cost $c_h$, the agent considers some $\hat{c}_{h+1} = f_h(\c_h, c_h)$ instead of $\c_{h+1} = \c_h + c_h$. 

We view $f$ as projecting a cumulative cost onto a smaller approximate cost space. Following the principle of optimism, we also ensure that $f(\c_h,c_h) \leq \c_h+c_h$. This guarantees that any optimal policy under the approximate costs achieves optimal value at the price of a slight violation in the budget.

\begin{algorithm}[t]
\caption{Approximate Reduction}\label{alg: approx-reduction}
\begin{algorithmic}[1]
\Require{cMDP $M$ and projection $f$}
\State $\hM \gets $ \cref{def: approx}$(M,f)$
\State $\pi, \hV^* \gets \text{Solve}(\hM)$
\If{ $\hV^* = -\infty$}
    \State \Return ``Infeasible"
\Else 
    \State \Return $\pi$
\EndIf
\end{algorithmic}
\end{algorithm}

If the agent records the approximate costs induced by the projection $f$, the interaction evolves according to the following unconstrained MDP.
\begin{definition}[Approximate MDP]\label{def: approx}  
The \emph{approximate MPD} $\hM := (\hS, \hA, \hP, \hR, H, \hs_0)$ where,
\begin{equation*}\label{equ: msa}
\begin{split}
    \hat{S}_{h+1} := \Big\{(s',\hat{c}')\mid \exists (s,\hat{c}) \in \hat{S}_h, a \in \A, c' \in C_h(s,a), \\
    \hat{c}' = f_h(\hat{c},c'), \; \Pr_{c \sim C_h(s,a)}[f_h(\hat{c}, c) \leq B] = 1, \; \\
    P_h(s' \mid s,a) > 0 \Big\},
\end{split}
\end{equation*}
is defined using approximate costs produced by safe exploration with a projection step. The other objects are defined analogously to \cref{def: augmentation}.
\end{definition}

Our approximation algorithms, \cref{alg: approx-reduction}, equate to solving $\hat{M}$ for different choices of $f$. 
To use any Markovian $\pi$ for $\hat{M}$, the agent just needs to apply $f$ when updating its approximate costs. In effect, the agent \emph{accumulates then projects} to create each approximate cost. The new interaction protocol is given by \cref{alg: approximation}.

To derive our choice of $f$, we first observe that the cumulative cost can never surpass $B$. Furthermore, should the agent ever accumulate a cost of $B - H\cmax$, it can no longer violate the budget along that trajectory. Thus, the agent's cumulative cost is always effectively within the interval $[B - H\cmax, B]$.

\paragraph{Projection.} Our approach is to evenly subdivide the interval $[B-H\cmax, B]$ by length-$\ell$ intervals centered around $0$. Then, the projection always maps a point in an interval to its left endpoint. Alternatively, we can think of $\ell$ as defining a new unit of measurement, and the projection maps each cumulative cost to its largest integer multiple of $\ell$ below the cumulative cost. Should the agent ever encounter an extremely negative cost, we safely truncate it to the projection of $B-(H-h)\cmax$. 

In symbols, we define our projection by, $f_{h}(\hat{c},c) := $
\begin{equation*}
    \begin{cases}
        \hat{c} + \floor{\frac{c}{\ell}}\ell & \text{if } \hat{c}+ c \geq  B - (H-h)\cmax\\
        \floor{\frac{B - (H-h)\cmax}{\ell}}\ell & \text{o.w.}\\
    \end{cases}
\end{equation*}
Critically, the projection is defined so that each approximate cost is an underestimate of the true cost, but no farther than $\epsilon$ away from the true cost (except when a cost smaller than $B-(H-h)\cmax$ is encountered). 

\begin{lemma}\label{lem: approx}
    For any feasible policy $\pi$ for $\hat{M}$ and any $h \in [H+1]$, $\P^{\pi}_M[(\hat{c}_h \leq \c_h \leq \hat{c}_h + (h-1)\ell) \lor (\c_h, \hat{c}_h \leq B-(H-h+1)\cmax)] = 1$.
    Also, $\abs{\set{\hat{c}_{h} \mid \exists s \in \S,  (s,\hat{c}_{h}) \in \hat{\S}_{h}}} \leq \paren{\frac{H\norm{\cmax}_{\infty}}{\ell} + 2}^d$.
\end{lemma}

We see that solving $\hM$ gives additive feasible solutions and $\hM$ has far fewer states than $\mM$.

\begin{theorem}[Approximation]\label{thm: approx}
    \cref{alg: approx-reduction} computes an $H\ell$-additive feasible policy whose value is at least the optimal value of \eqref{equ: objective} and that can be stored with $O\paren{H^{d+1}S\norm{\cmax}_{\infty}^d/\ell^d}$ space.
\end{theorem}

\begin{algorithm}[t]
\caption{Approximate Interaction Protocol}\label{alg: approximation}
\textbf{Input:} policy $\pi$ and projection $f$
\begin{algorithmic}[1]
\State $\hat{s}_1 = (s_0, 0)$ and $\hat{c}_1 = 0$.
\For{$h = 1$ to $H$}
    \State $a_h = \pi_h(\hat{s}_h)$
    \State $c_h \sim C_h(s_h,a_h)$ and $s_{h+1} \sim P_h(s_h,a_h)$
    \State $\hat{c}_{h+1} = f_h(\hat{c}_h,c_h)$
    \State $\hat{s}_{h+1} = (s_{h+1}, \hat{c}_{h+1})$
\EndFor
\end{algorithmic}
\end{algorithm}

Like with our original reduction, the interaction protocol in \cref{alg: approximation} allows the agent to simulate $\hat{M}$ online through $M$. Thus, planning and learning in $\hat{M}$ can be done through $M$. 

\begin{remark}
    Note, the agent does not need to construct $\hat{S}$ using \cref{equ: msa}; it suffices to consider the finite, stationary state space $\S \times \mathcal{C}$, where $\mathcal{C}$ is the $\ell$-cover of $[B-H\cmax, B]$ defined by $f$. Technically, for learning, the agent should already know or have learned a bound on $\cmax$ to know the approximate state space.
\end{remark}

\subsection{Approximation Guarantees}
\label{sec:approx}
We can use \cref{alg: approx-reduction} with different choices of $\ell$ to achieve the traditional approximation guarantees defined in \cref{def: approx-feasibility}.

\paragraph{Additive Approximation.} Given any $\epsilon > 0$, we can compute an $\epsilon$-additive feasible solution by choosing $\ell := \frac{\epsilon}{H}$. This approach is efficient so long as $\cmax$ is not too large, since $\cmax$ controls the number of discretized costs we need to consider.

\begin{corollary}[Additive Reduction]\label{cor: additive}
    For any $\epsilon > 0$, an optimal value, $\epsilon$-additive feasible policy for an anytime-constrained cMDP can be computed in $O\paren{H^{4d+1} S^2 A \norm{\cmax}_{\infty}^{2d}/\epsilon^{2d}}$ time and learned with $\tilde O\paren{H^{2d+3} SA \norm{\cmax}_{\infty}^d \allowbreak \log(\frac{1}{\delta})/\gamma^2\epsilon^d}$ sample complexity using \cref{alg: approx-reduction} with $\ell := \frac{\epsilon}{H}$. Furthermore, the amount of space needed to store the policy is $O\paren{H^{2d+1}S\norm{\cmax}_{\infty}^d/\epsilon^d}$. Thus, if d is constant and $\cmax \leq \poly(|M|)$, our additive methods are polynomial time and sample complexity.
\end{corollary}

\paragraph{Relative Approximation.} Given any $\epsilon > 0$, we can compute an $\epsilon$-relative feasible solution by choosing $\ell := \frac{\epsilon\abs{B}}{H}$. This approach is efficient so long as $\cmax$ is not much larger than $\abs{B}$. This allows us to capture cost ranges that are polynomial multiples of $\abs{B}$, which could be exponentially large, unlike the additive approximation which requires that $\cmax$ to be polynomial.

\begin{corollary}[Relative Reduction]\label{cor: relative}
    For any $\epsilon > 0$, if $\cmax \leq x\abs{B}$, an optimal value, $\epsilon$-relative feasible policy for an anytime-constrained cMDP can be computed in $O\paren{x^{2d}H^{4d+1} S^2 A/\epsilon^{2d}}$ time and learned with $\tilde O\paren{x^d H^{2d+3} SA/\epsilon^d \log(\frac{1}{\delta})/\gamma^2}$ sample complexity using \cref{alg: approx-reduction} with $\ell = \frac{\epsilon\abs{B}}{H}$. Furthermore, the amount of space needed to store the policy is $O\paren{x^d H^{2d+1}S/\epsilon^d}$. 
    Thus, if $d$ is constant and $x \leq \poly(|M|)$, our methods are polynomial time and sample complexity.
\end{corollary}

\begin{corollary}\label{cor: improved}
    If all costs are positive, the $H$ dependence in each guarantee of both the additive and relative approximation improves to $H^{2d+1}$, $H^{d+3}$, and $H^{d+1}$, respectively.
\end{corollary}

\paragraph{Limitations.}
Using our additive approximation, we can efficiently handle any cMDP with $\cmax \leq \poly(|M|)$. Using the relative approximation, we can even handle the case that $\cmax$ is exponentially large so long as $\cmax \leq \poly(|M|)\abs{B}$. Thus, we can efficiently compute approximately feasible solutions so long as $\cmax \leq \poly(|M|)\max(1,\abs{B})$. 

We point out that the condition $\cmax \leq \poly(|M|)\abs{B}$ is very natural. If the costs all have the same sign, any feasible policy induces costs with $\cmax \leq |B|$. In our driving example, the condition simply says the vehicle cannot use more fuel than the capacity of the tank. In fact, this bottleneck is not due to our approach; some bound on $\cmax$ is necessary for efficient computation as \cref{prop: relax-hard} shows.

\begin{proposition}\label{prop: relax-hard}
    For any fixed $\epsilon > 0$, computing an optimal-value, $\epsilon$-additive or $\epsilon$-relative feasible solution to the knapsack problem with negative weights is NP-hard. Hence, it is hard for anytime-constrained cMDPs. 
\end{proposition}

\paragraph{Feasibility Scheme.} 
Let $OPT(B)$ denote the optimal value $V^*$ of an anytime-constrained cMDP with budget $B$.
\cref{alg: approx-reduction} provides an efficient approximation and guarantees at least $OPT(B)$ value but with the possibility of slightly going over budget, up to $B+\epsilon$ or $B(1+\epsilon)$.
If it is important that the budget $B$ is never violated, we can use the same approximate algorithm with one change: 
We instead give it $\hat{M}'$, which is $\hat{M}$ constructed from $M$ under a smaller budget: (1) $B-\epsilon$ for the additive approximation and (2) $B/(1+\epsilon)$ for the relative approximation.
Then, any over-budget by \cref{alg: approx-reduction} is compensated by the smaller budget, so that the cumulative cost is still under $B$.
Thus we have both efficiency and (budget $B$) feasibility.
The drawback is that the algorithm now only guarantees a value at least $OPT(B-\epsilon)$ or $OPT(B/(1+\epsilon))$, both can be much smaller than $OPT(B)$.

\begin{proposition}[Feasible Solutions]\label{prop: feasible-scheme}
    If $\pi$ is returned by \cref{alg: approx-reduction} using $\ell = \frac{\epsilon}{H}$ ($\ell = \frac{\epsilon\abs{B}}{H}$) with budget $B' = B-\epsilon$ ($B' = \frac{B}{1+\epsilon}$), then $\pi$ is feasible for \eqref{equ: objective}.
\end{proposition}

\section{EXPERIMENTS}\label{sec: experiments}

We test our methods on the family of NP-hard cMDP instances that we constructed in the proof of \cref{thm: np-hardness}. Namely, cMDPs with one state, $\S = \{0\}$, two actions, $\A = \{0,1\}$, and a single constraint, $d = 1$. The rewards and costs satisfy $r_h(0,0) = c_h(0,0) = 0$ for all $h$. For action $1$, $r_h(0,1) = x_h$ and $c_h(0,1) = y_h$ where for all $h$, $x_h,y_h\sim \textrm{Unif}[0,1]$ are chosen and then fixed for each cMDP instance. Since the complexity blowup for anytime constraints is in the time horizon, we focus on varying the time horizon. 

Already for a cMDP with $H = B = 15$, our exact reduction takes around a minute to run, which is unsurprising as the complexity grows like $30 (2^{15})^2$. Instead of comparing to the reduction, we can instead compare our relative approximation (\cref{cor: relative}) to our feasibility scheme (\cref{prop: feasible-scheme}). By definition, the relative approximation achieves at least the optimal value and the feasibility scheme is feasible. Thus, if both algorithms achieve value close to each other, we know that the approximation is not violating the budget much and the feasibility scheme is achieving near-optimal value. 

We perform $N = 10$ trials for each $H \in \{10, 20, 30, 40, 50\}$. We consider two different budgets, $b \in \{.1, 10\}$, and $\epsilon = 0.1$. We report the value of the worst trial (the one where both methods are farthest apart) for each $H$ in \cref{fig: value}. We see that both values are consistently close even in the worst case, which indicates the feasible solution is nearly optimal and the approximate solution is nearly feasible. 

To test the scaling of our approximation algorithm, we ran  $N = 10$ trials for each $H \in \set{10, 20, \ldots, 100}$. Here, we use a budget of $b = 100$ to maximize the cost diversity. This time, we tried both $\epsilon = 0.1$ and the even larger $\epsilon = 1$. We report the worst running time from each $H$ under $\epsilon = 0.1$ in \cref{fig: scale}. We see a cubic-like growth as expected from \cref{cor: improved}. Also, the approximation easily handled a large horizon of $100$ in a mere $2$ seconds, which drastically beats out the exact reduction. 
For $\epsilon = 1$ the results are even more striking with a maximum run time of $.0008$ seconds and the solutions are guaranteed to violate the budget by no more than a multiple of $2$! 
We give additional details and results in the Appendix.

\begin{figure}
\centering
  \includegraphics[scale=.45]{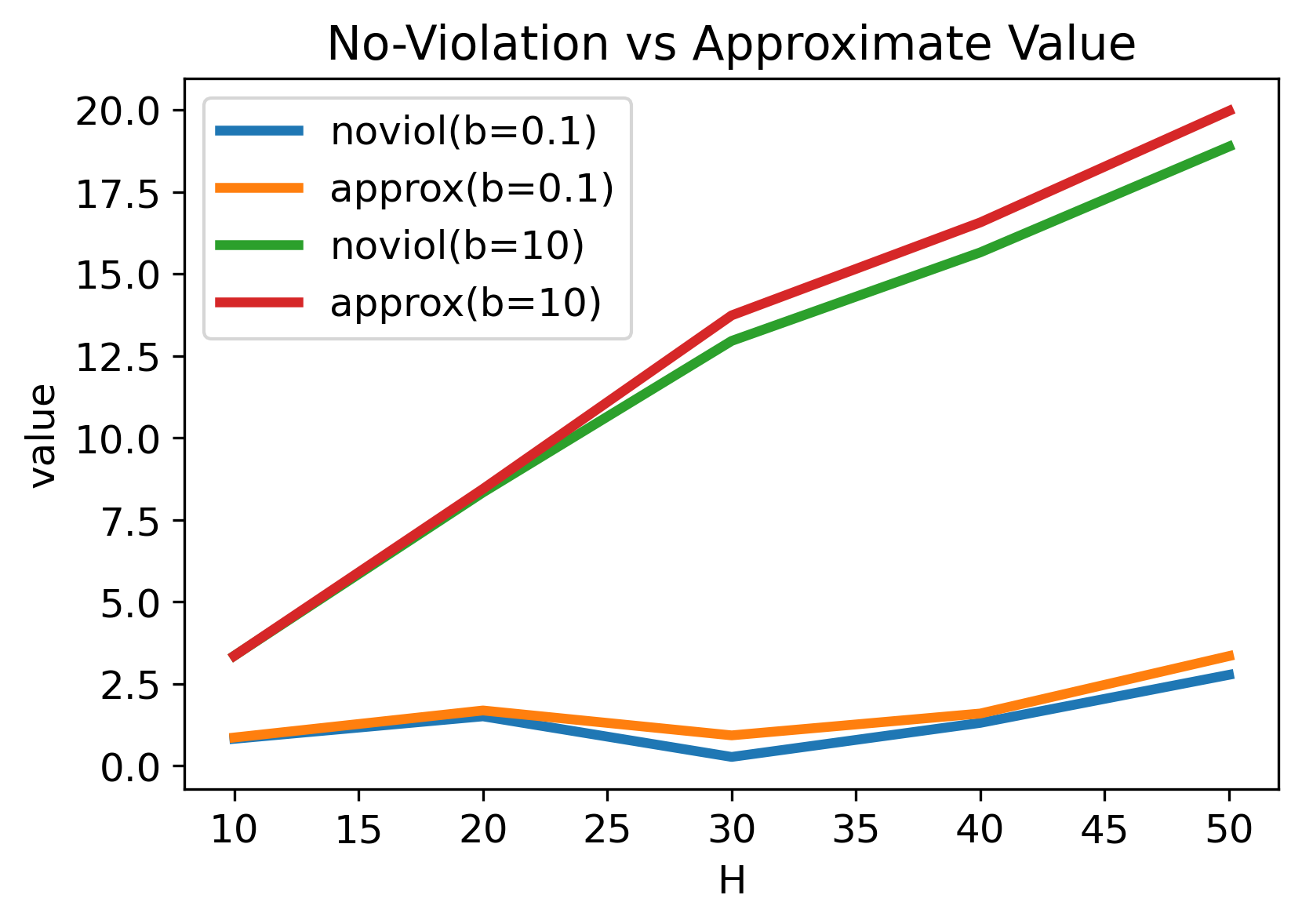}
  \caption{Value comparison of our relative approximation and feasibility scheme.}
  \label{fig: value}
\end{figure}
\begin{figure}
\centering
  \includegraphics[scale=.45]{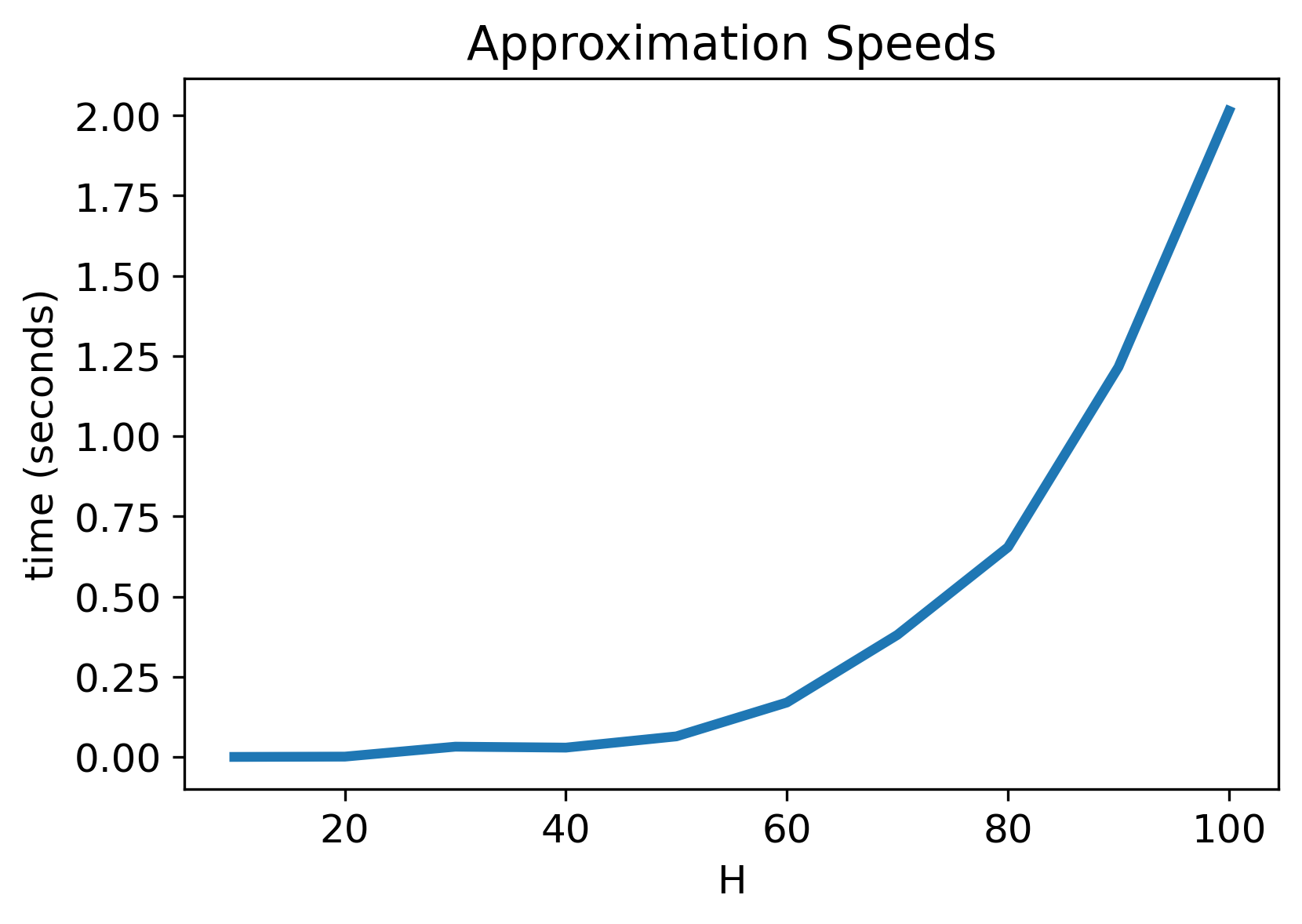}
  \caption{Running time of our relative approximation.}
  \label{fig: scale}
\end{figure}

\section{CONCLUSIONS}\label{sec: conclusion}

In this paper, we formalized and rigorously studied anytime-constrained cMDPs. Although traditional policies cannot solve anytime-constrained cMDPs, we showed that deterministic augmented policies suffice. We also presented a fixed-parameter tractable reduction based on cost augmentation and safe exploration that yields efficient planning and learning algorithms when the cost precision is $O(\log(|M|))$. In addition, we developed efficient planning and learning algorithms to find $\epsilon$-approximately feasible policies with optimal value whenever the maximum supported cost is $O(\poly(|M|)\max(1, |B|))$. Based on our hardness of approximation results, this is the best approximation guarantee we can hope for under worst-case analysis.

Although we have resolved many questions, there are still many more mysteries about anytime constraints. Primarily, we focus on worst-case analysis, which may be too pessimistic. Since anytime constraints are so sensitive to changes in cost, a smoothed or average case analysis could be promising. Going further, there may be useful classes of cMDPs for which \eqref{equ: objective} is efficiently solvable even in the worst case. Proving lower bounds on the exact complexity of computing solutions is also interesting. Lastly, learning a feasible policy without violation during the learning process is an important open question.

\subsubsection*{Acknowledgments}
This project is supported in part by NSF grants 1836978, 2023239, 2202457, 2331669, ARO MURI W911NF2110317, AF CoE FA9550-18-1-0166, and DMS-2023239.

\bibliographystyle{abbrvnat}
\bibliography{references}

\section*{Checklist}

 \begin{enumerate}

 \item For all models and algorithms presented, check if you include:
 \begin{enumerate}
   \item A clear description of the mathematical setting, assumptions, algorithm, and/or model. [Yes]
   \item An analysis of the properties and complexity (time, space, sample size) of any algorithm. [Yes]
   \item (Optional) Anonymized source code, with specification of all dependencies, including external libraries. [Yes]
 \end{enumerate}

 \item For any theoretical claim, check if you include:
 \begin{enumerate}
   \item Statements of the full set of assumptions of all theoretical results. [Yes]
   \item Complete proofs of all theoretical results. [Yes]
   \item Clear explanations of any assumptions. [Yes]     
 \end{enumerate}

 \item For all figures and tables that present empirical results, check if you include:
 \begin{enumerate}
   \item The code, data, and instructions needed to reproduce the main experimental results (either in the supplemental material or as a URL). [Yes]
   \item All the training details (e.g., data splits, hyperparameters, how they were chosen). [Yes]
         \item A clear definition of the specific measure or statistics and error bars (e.g., with respect to the random seed after running experiments multiple times). [Yes]
         \item A description of the computing infrastructure used. (e.g., type of GPUs, internal cluster, or cloud provider). [Yes]
 \end{enumerate}

 \item If you are using existing assets (e.g., code, data, models) or curating/releasing new assets, check if you include:
 \begin{enumerate}
   \item Citations of the creator If your work uses existing assets. [Not Applicable]
   \item The license information of the assets, if applicable. [Not Applicable]
   \item New assets either in the supplemental material or as a URL, if applicable. [Not Applicable]
   \item Information about consent from data providers/curators. [Not Applicable]
   \item Discussion of sensible content if applicable, e.g., personally identifiable information or offensive content. [Not Applicable]
 \end{enumerate}

 \item If you used crowdsourcing or conducted research with human subjects, check if you include:
 \begin{enumerate}
   \item The full text of instructions given to participants and screenshots. [Not Applicable]
   \item Descriptions of potential participant risks, with links to Institutional Review Board (IRB) approvals if applicable. [Not Applicable]
   \item The estimated hourly wage paid to participants and the total amount spent on participant compensation. [Not Applicable]
 \end{enumerate}

 \end{enumerate}

\onecolumn
\appendix

\section{Proofs for \texorpdfstring{\cref{sec: constraints}}{sec: constraints}}

\subsection{Proof of \texorpdfstring{\cref{prop: Markovian-suboptimality}}{subsec: proposition 1}}

\begin{proof}
    Fix any $H \geq 2$. For any $h \in [H-1]$, let $\pi$ be a cost-history-dependent policy that does not record the cost at time $h$. For any such $\pi$, we construct a cMDP instance for which $\pi$ is arbitrarily suboptimal. This shows that any class of policies that does not consider the full cost history is insufficient to solve \eqref{equ: objective}. In particular, the class of Markovian policies does not suffice.
    
    Consider the simple cMDP $M_h$ defined by a single state, $\S = \{0\}$, two actions, $\A = \{0,1\}$, and horizon $H$. The initial state is trivially $0$ and the transitions are trivially self-loops from $0$ to $0$. Importantly, $M_h$ has non-stationary rewards and costs. The rewards are deterministic. For any $t \not = h+1$, $r_t(s,a) = 0$. For some large $x > 0$, $r_{h+1}(s, 1) = x$ and $r_{h+1}(s,0) = 0$. The costs are deterministic except at time $h$. For any $t \not \in \{h, h+1\}$, $c_t(s,a) = 0$. For $t = h+1$, $c_{h+1}(s,1) = B$ and $c_{h+1}(s,0) = 0$. For $t = h$, the costs are random: 
    \begin{align*}
        C_h(s,a) := \begin{cases}
            B & \text{w.p. } \frac{1}{2} \\
            0 & \text{w.p. } \frac{1}{2}
        \end{cases}
    \end{align*}
    The budget is any $B > 0$.
    
    Clearly, an optimal cost-history-dependent policy can choose any action it likes other than at time $h+1$. At time $h+1$, an optimal policy chooses $a = 1$ if the cost incurred at time $h$ was $0$ and otherwise chooses action $a = 0$. The value of the optimal policy is $x/2$ since the agent receives total reward $x$ whenever $c_h = 0$, which is half the time, and otherwise receives total reward $0$. Thus, $V^*_{M_h} = x/2$. 
    
    On the other hand, consider $\pi$'s performance. Since $\pi$ does not record $c_h$, it cannot use $c_h$ to make decisions. Hence, $p := \P^{\pi}_{M_h}[a_{h+1} = 1]$ is independent of $c_h$'s value. If $p > 0$ then with probability $1/2 p > 0$, the agent accrues cost $B$ at both time $h$ and time $h+1$ so violates the constraint. Thus, if $\pi$ is feasible, it must satisfy $p = 0$. Consequently, $\pi$ can never choose $a = 1$ at time $h+1$ and so can never receive rewards other than $0$. Thus, $V^{\pi}_{M_h} = 0 << x/2 = V^*_{M_h}$. By choosing $x$ large enough, we see policies that do not consider the entire cost history can be arbitrarily suboptimal. By applying this argument to $H = 2$ and $h = 1$, we see that Markovian policies can be arbitrarily suboptimal and so do not suffice to solve anytime-constrained cMDPs.
\end{proof}

\subsection{Proof of \texorpdfstring{\cref{cor: cMDP-failure}}{subsec: corollary 1}}

\begin{proof}
    Since optimal policies for expectation-constrained cMDPs are always Markovian, \cref{prop: Markovian-suboptimality} immediately implies such policies are infeasible or arbitrarily suboptimal. In fact, we can see this using the same construction of $M_h$. 

    \begin{enumerate}
        \item Under an expectation constraint, the optimal policy $\pi$ can choose $p = 1/2$ and still maintain that $\E^{\pi}_M[\sum_{t = 1}^H c_t] = B/2 + pB = B \leq B$. Thus, such a policy violates the anytime constraint by accumulating cost $2B$ with probability $1/4$. In fact, if we generalize the construction of $M_h$ to have $c_h = \frac{B}{2\delta}$ with probability $\delta > 0$ (where $\delta = 1/2$ in the original construction), then the optimal expectation-constrained policy is the same $\pi$ but now accumulates cost $\frac{B}{2\delta} + B$ with probability $\delta/2 > 0$. Since $\delta$ can be chosen to be arbitrarily small, the violation of the anytime constraint, which is $B/2\delta$, can be arbitrarily large. Even if we relax the policy to just be $\epsilon$-optimal, for any $\epsilon > 0$ we can choose $x$ large enough to where all $\epsilon$-optimal policies still select action $1$ with non-zero probability. 
        
        \item A similar construction immediately shows the arbitrarily infeasibility of optimal chance-constrained policies. Consider a chance constraint that requires $\P^{\pi}_M[\sum_{t = 1}^H c_t > B] \leq \delta$ for some $\delta > 0$. We can use the same construction as above but with an arbitrarily larger cost of $c_h = y\frac{B}{2\delta}$ for some $y > 0$. Then, an optimal chance constrained policy can always let $p = 1$ since the cost only exceeds budget when $c_h > 0$ which happens with probability $\delta$. Such a policy clearly violates the anytime constraint by $y\frac{B}{2\delta}$, which is arbitrarily large by choosing $y$ to be arbitrarily large. Also, observe this does not require us to consider Markovian policies since whether the budget was already violated at time $h$ or not, the policy is still incentivized to choose action $1$ at time $h+1$ as additional violation does not effect a chance-constraint. Again, considering an $\epsilon$-optimal policy does not change the result.
        \item 
    \end{enumerate}

    Suppose instead we computed an optimal policy using a smaller budget $B'$.
    \begin{enumerate}
        \item For expectation-constraints, to ensure the resultant policy is feasible for anytime constraints, we need that $p = 0$ as before. By inspection, it must be that $B' = B/2$ but then the value of the policy is $0$ which is arbitrarily suboptimal as we saw before. 

        \item For chance-constraints, the situation is even worse. Consider the $M_h$ but with $c_h = B$ w.p. $\delta$. Then, no matter what $B'$ we choose, the resultant policy is not feasible. Specifically, an optimal cost-history-dependent policy under the event that $c_h = B/2$ will then choose $a_{h+1} = 1$ almost surely since the extent of the violation does not matter. But even ignoring this issue, under the event that $c_h = 0$ the policy would then have to choose $a_{h+1} = 0$ which is again arbitrarily suboptimal. 
    \end{enumerate}

    For the knapsack-constrained frameworks, the policy is allowed to violate the budget arbitrarily once per episode. Thus, no matter how we change the budget feasibility is never guaranteed: it will always choose $a_{h+1} = 1$ in any realization. The other frameworks also fail using slight modifications of the constructions above.

\end{proof}

\subsection{Proof of \texorpdfstring{\cref{prop: cost-sensitivity}}{subsec: proposition 2}}

\begin{proof}
    For continuity with respect to rewards, notice that if a certain reward is not involved in the optimal solution, then any perturbation does not change $V^*$. On the other hand, if a reward is in an optimal solution, since $V^*$ is defined by an expectation of the rewards, it is clear that $V^*$ is continuous in that reward: a slight perturbation in the reward leads to the same or an even smaller perturbation in $V^*$ due to the probability weighting.

    On the other hand, $V^*$ can be highly discontinuous in $c$ and $B$. Suppose a cMDP has only one state and two actions with cost $0$ and $B$, and reward $0$ and $x \in \Real_{> 0}$ respectively. Then slightly increasing the cost or slightly decreasing the budget to create a new instance $M_{\epsilon}$ moves a solution value of $x$ all the way down to a solution value of $0$. In particular, we see $V^*_{M} = x >> V^*_{M_{\epsilon}}$ if we only perturb the budget slightly by some $\epsilon > 0$.
\end{proof}

\subsection{Proof of \texorpdfstring{\cref{lem: derandomization}}{subsec: lemma 1}}

We first formally define the anytime cost of a policy $\pi$ as,
\begin{equation*}
    C^{\pi} := \max_{h \in [H]} \max_{\substack{\tau_h \in \H_h, \\ \P^{\pi}[\tau_h] > 0}} \c_h.
\end{equation*}
In words, $C^{\pi}$ is the largest cost the agent ever accumulates at any time under any history.

\begin{proof}

Consider the deterministic policy $\pi'$ defined by,
\[\pi'_h(\tau_h) := \max_{\substack{a \in \A, \\ \pi_h(a \mid \tau_h) > 0}} r_h(s,a) + \E_{c, s'} \brac{V_{h+1}^{\pi}(\tau_h, a, c, s')},\]
for every $h \in [H]$ and every $\tau_h \in \H_h$.

We first show that for any $\tau_h \in \H_h$, if $\P^{\pi'}[\tau_h] > 0$ then $\P^{\pi}[\tau_h] > 0$. This means that the set of partial histories induced by $\pi'$ with non-zero probability are a subset of those induced by $\pi$. Hence, 
\[C^{\pi'} = \max_{h \in [H]}\max_{\substack{\tau_h \in \H_h, \\ \P^{\pi'}[\tau_h] > 0}} \c_h \leq \max_{h \in [H]}\max_{\substack{\tau_h \in \H_h, \\ \P^{\pi}[\tau_h] > 0}} \c_h = C^{\pi}.\]

We show the claim using induction on $h$. For the base case, we consider $h = 1$. By definition, we know that for both policies, $\P^{\pi'}[s_0] = \P^{\pi}[s_0] = 1$. For the inductive step, consider any $h \geq 1$ and suppose that $\P^{\pi'}[\tau_{h+1}] > 0$. Decompose $\tau_{h+1}$ into $\tau_{h+1} = (\tau_h, a, c, s')$ and let $s = s_h$. As we have shown many times before,
\begin{align*}
    0 < \P^{\pi'}[\tau_{h+1}] &=  \pi'_h(a \mid \tau_h)C_h(c \mid s, a)P_h(s' \mid s,a)\P^{\pi'}[\tau_h]
\end{align*}
Thus, it must be the case that $\pi'_h(\tau_h) = a$ (since $\pi'$ is deterministic), $C_h(c \mid s,a) > 0$, $P_h(s' \mid s,a) > 0$, and $\P^{\pi'}[\tau_h] > 0$. By the induction hypothesis, we then know that $\P^{\pi}[\tau_h] > 0$. Since by definition $\pi'_h(\tau_h) = a \in \set{a' \in \A \mid \pi_h(a' \mid \tau_h) > 0}$, we then see that,
\begin{align*}
    \P^{\pi}[\tau_{h+1}] = \pi_h(a \mid \tau_h)C_h(c \mid s, a)P_h(s' \mid s,a)\P^{\pi}[\tau_h] > 0
\end{align*}
This completes the induction.

Next, we show that for any $h \in [H]$ and $\tau_h \in \H_h$, $V^{\pi'}_h(\tau_h) \geq V^{\pi}_h(\tau_h)$. This implies that $V^{\pi'}_M = V^{\pi'}_1(s_0) \geq V^{\pi}_1(s_0) = V^{\pi}_M$ which proves the second claim. We proceed by backward induction on $h$. For the base case, we consider $h = H+1$. By definition, both policies achieve value $V_{H+1}^{\pi'}(\tau) = 0 = V_{H+1}^{\pi}(\tau)$. For the inductive step, consider $h \leq H$. By \eqref{equ: PE},
\begin{align*}
    V_h^{\pi'}(\tau_h) &= r_h(s,\pi'(\tau_h)) + \E_{c,s'}[V^{\pi'}_{h+1}(\tau_{h+1})] \\
    &\geq r_h(s,\pi'(\tau_h)) + \E_{c,s'}[V^{\pi}_{h+1}(\tau_{h+1})] \\
    &\geq \E_{a}[r_h(s,a) + \E_{\Tilde{c},s'}[V^{\pi}_{h+1}(\tau_{h+1})]] \\
    &= V_h^{\pi}(\tau_h).
\end{align*}
The second line used the induction hypotheses. The third lines used the fact that the maximum value is at least any weighted average. This completes the induction.

Thus, we see that $\pi'$ satisfies $C^{\pi'}_M \leq C^{\pi}_M$ and $V^{\pi'}_M \geq V^{\pi}_M$ as was to be shown. Furthermore, we see that $\pi'$ can be computed from $\pi$ in linear time in the size of $\pi$ by just computing $V^{\pi}_h(\tau_h)$ by backward induction and then directly computing a solution for each partial history.

\end{proof}

\subsection{Proof of \texorpdfstring{\cref{thm: np-hardness}}{subsec: theorem 1}}

\begin{proof}
    We present a poly-time reduction from the knapsack problem. Suppose we are given $n$ items each with a non-negative integer value $v_i$ and weight $w_i$. Let $B$ denote the budget. We construct an MDP $M$ with $\S = \{0\}$, $\A = \{0, 1\}$, and $H = n$. Naturally, having a single state implies the initial state is $s_0 = 0$, and the transition is just a self-loop: $P(0 \mid 0, a) = 1$ for any $a \in \A$. The rewards of $M$ correspond to the knapsack values: $r_i(s,1) = v_i$. The costs of $M$ correspond to the knapsack weights: $c_i(s,1) = w_i$. The budget remains $B$. 

    Clearly, any (possibly non-stationary) deterministic policy corresponds to a choice of items for the knapsack. By definition of the rewards and costs, $\pi_h(\cdot) = 1$ if and only if the agent gets reward $v_h$ and accrues cost $c_h$. Thus, there exists a deterministic $\pi \in \Pi$ with $V^{\pi}_M \geq V$ and $C^{\pi}_M \leq B$ if and only if $\exists I \subseteq [n]$ with $\sum_{i \in I} v_i \geq V$ and $\sum_{i \in I} w_i \leq B$. From \cref{lem: derandomization} if there exists a stochastic optimal policy for the cMDP with value at least $V$ and anytime cost at most $B$, then there exists a deterministic policy with value at least $V$ and anytime cost at most $B$. As $M$ can be constructed in linear time from the knapsack instance, the reduction is complete.
\end{proof}

\subsection{Proof of \texorpdfstring{\cref{thm: approx-hard}}{subsec: theorem 2}}

\begin{proof}
    We show that computing a feasible policies for anytime constrained cMDPs with only $d = 2$ constraints is NP-hard via a reduction from Partition. Suppose $X = \{x_1, \ldots, x_n\}$ is a set of non-negative integers. Let $Sum(X) := \sum_{i = 1}^n x_i$. We define a simple cMDP similar to the one in the proof of \cref{thm: np-hardness}. Again, we define $\S = \{0\}$, $\A = \{0,1\}$, and $H = n$. The cost function is deterministic, defined by $c_{i,h}(s,i) = x_h$ and $c_{i,h}(s, 1-i) = 0$. The budgets are $B_0 = B_1 = Sum(X)/2$. 

    Intuitively, at time $h$, choosing action $a_h = 0$ corresponds to placing $x_h$ in the left side of the partition and $a_h = 1$ corresponds to placing $x_h$ in the right side. The total cumulative cost for each constraint corresponds to the sum of elements in each side of the partition. If both sides sum to at most $Sum(X)/2$ then it must be the case that both are exactly $Sum(X)/2$ and so we have found a solution to the Partition problem.

    Formally, we show that $\exists \pi \in \Pi_M$ if and only if $\exists Y \subseteq [n]$ with $Sum(Y) = Sum(Z) = Sum(X)/2$ where $Z = X \setminus Y$.

    \begin{itemize}
        \item ($\implies$) Suppose $\pi$ is a feasible deterministic policy for $M$ (We can assume deterministic again by \cref{lem: derandomization}). Define $Y := \{i \mid \pi_h(s) = 0\}$ and $Z := \{i \mid \pi_h(s) = 1\}$. Since $\pi$ is deterministic we know that each item is assigned to one set or the other and so $Y \cup Z = X$. 
        
        By definition of the constraints, we have that $\P^{\pi}_M[\sum_{h = 1}^H c_{i,h} \leq B_i] = 1$. Since all quantities are deterministic this means that $\sum_{h = 1}^H c_{i,h} \leq B_i$. By definition of $Y$ and $Z$ we further see that $Sum(Y) = \sum_{h = 1}^H c_{0,h} \leq Sum(X)/2$ and $Sum(Z) = \sum_{h = 1}^H c_{1,h} \leq Sum(X)/2$. Since,
        \[Sum(X) = Sum(Y \cup Z) = Sum(Y) + Sum(Z) \leq Sum(X)/2 + Sum(X)/2 = Sum(X),\] 
        the inequality must be an equality. Using $Sum(Y) = Sum(X) - Sum(Z)$, then implies that $Sum(Y) = Sum(Z) = Sum(X)/2$ and so $(Y,Z)$ is a solution to the partition problem.

        \item ($\impliedby$) On the other hand, suppose that $(Y,Z)$ is a solution to the partition problem. We can define $\pi_h(s) = 0$ if $h \in Y$ and $\pi_h(s) = 1$ if $h \in Z$. By definition, we see that $Sum(Y) = \sum_{h = 1}^H c_{0,h} = Sum(X)/2 = B_0$ and $Sum(Z) = \sum_{h = 1}^H c_{1,h} = Sum(X)/2 = B_1$. Thus, $\pi$ is feasible for $M$.
    \end{itemize}

    As the construction of $M$ can clearly be done in linear time by copying the costs and computing $Sum(X)/2$, the reduction is polynomial time. Thus, it is NP-hard to compute a feasible policy for an anytime-constrained cMDP.

    Since the feasibility problem is NP-hard, it is easy to see approximating the problem is NP-hard by simply defining a reward of $1$ at the last time step. Then, if an approximation algorithm yields any finite-value policy, we know it must be feasible since infeasible policies yield $-\infty$ value. Thus, any non-trivial approximately-optimal policy to an anytime-constrained cMDP is NP-hard to compute.
\end{proof}

\section{Proofs for \texorpdfstring{\cref{sec: reduction}}{sec: reduction}}

The complete forward induction algorithm that computes $\mS$ as defined in \cref{def: augmentation} is given in \cref{alg: cost-states}. 

\begin{algorithm}[tb]
\caption{Compute $\{\mS_h\}_h$ }\label{alg: cost-states}
\textbf{Input: } $(M,C,B)$
\begin{algorithmic}
\State $\mS_1 = \set{(s_0,0)}$
\For{$h \gets 1$ to $H-1$}
    \State $\mS_{h+1} = \varnothing$
    \For{$(s,\c) \in \mS_{h}$}
        \For{$s' \in \S$}
            \For{$a \in \A$}
                \If{$P_{h}(s' \mid s,a) > 0$ and $\Pr_{c \sim C_h(s,a)}[\c + c \leq B] = 1$}
                    \For{$c \in C_h(s,a)$}
                        \State $\mS_{h+1} \gets \mS_{h+1} \cup \set{(s', \c + c)}$
                    \EndFor
                \EndIf
            \EndFor
        \EndFor
    \EndFor
\EndFor
\State \Return $\{\mS_h\}_h$.
\end{algorithmic}
\end{algorithm}

Suppose $\tau_{h+1} \in \H_{h+1}$ is any partial history satisfying $\P^{\pi}_{\tau_h}[\tau_{h+1}] > 0$. By the Markov property (Equation (2.1.11) from \cite{MDP-book}), we have that
\begin{equation}\label{equ: mp}\tag{MP}
    \P^{\pi}_{\tau_h}[\tau_{h+1}] = \pi_h(a \mid \tau_h) C_h(c \mid s,a) P_h(s' \mid s,a).
\end{equation}
Thus, it must be the case that $\tau_{h+1} = (\tau_h, a, c, s')$ where $\pi_h(a \mid \tau_h) > 0$, $C_h(c \mid s,a) > 0$, and $P_h(s' \mid s,a) > 0$.

\subsection{Proof of \texorpdfstring{\cref{lem: feasibility}}{subsec: Lemma 2}}

We show an alternative characterization of the safe exploration state set:
\begin{equation}
\begin{split}
    SE_h := \Big\{(s,\c) \mid \exists \pi \; \exists \tau_h \in \H_h, \; &\P^{\pi}_{\tau_h}[s_h = s, \c_h = \c] = 1 \text{ and } \\
    &\forall k \in [h-1] \; \P^{\pi}_{\tau_k}[\c_{k+1} \leq B] = 1 \Big \}.
\end{split}
\end{equation}
Observe that $\P^{\pi}_{\tau_h}[s_h = s,\c_h = \c] = 1$ is equivalent to requiring that for $\tau_h$, $s_h = s$, $\c_h = \c$, and $\P^{\pi}[\tau_h] > 0$. 

\begin{lemma}\label{lem: safe-exploration}
    For all $h \in [H+1]$, $\mS_h = SE_h$.
\end{lemma}

We break the proof into two claims. 

\begin{claim}\label{claim: mS-in-SE}
    For all $h \in [H+1]$, $\mS_h \subseteq SE_h$.
\end{claim}

\begin{proof}
    We proceed by induction on $h$. For the base case, we consider $h = 1$. By definition, $\mS_1 = \set{(s_0,0)}$. For $\tau_1 = s_0$, we have $\c_1 = 0$ is an empty sum. Thus, for any $\pi \in \Pi$, $\P^{\pi}[s_1 = s_0, \c_1 = 0 \mid \tau_1] = 1$. Also, $[h-1] = [0] = \varnothing$ and so the second condition vacuously holds. Hence, $(s_0,0) \in SE_1$ implying that $\mS_1 \subseteq SE_1$.

    For the inductive step, we consider any $h \geq 1$. Let $(s',\c') \in \mS_{h+1}$. By definition, we know that there exists some $(s,\c) \in \mS_h$, $a \in \A$, and $c \in \Real$ satisfying,
    \[C_h(c \mid s,a) > 0, \; \Pr_{c \sim C_h(s,a)}[\c + c \leq B] = 1, \; \c' = \c + c, \text{ and } P_h(s' \mid s,a) > 0.\]
    By the induction hypothesis, $(s,\c) \in SE_h$ and so there also exists some $\pi \in \Pi$ and some $\tau_h \in \H_h$ satisfying,
    \[\P^{\pi}_{\tau_h}[s_h = s,\c_h = \c] = 1 \text{ and } \forall k \in [h-1], \; \P^{\pi}_{\tau_k}[\c_{k+1} \leq B] = 1.\]
    Overwrite $\pi_h(\tau_h) = a$ and define $\tau_{h+1} = (\tau_h,a,c,s')$. Then, by definition of the interaction with $M$, $\P^{\pi}[\tau_{h+1}] \geq \P^{\pi}[\tau_h]\pi_h(a\mid\tau_h) C_h(c \mid s,a) P_h(s' \mid s,a) > 0$. Here we used the fact that if $\P^{\pi}_{\tau_h}[s_h = s,\c_h = \c] = 1$ then $\P^{\pi}[\tau_h] > 0$ by the definition of conditional probability. Thus, $\P^{\pi}_{\tau_{h+1}}[s_{h+1} = s', \c_{h+1} = \c'] = 1$. By assumption, $\P^{\pi}_{\tau_k}[\c_{k+1} \leq B]$ holds for all $k \in [h-1]$. For $k = h$, we have
    \begin{align*}
        \P^{\pi}_{\tau_h}[\c_{h+1} \leq B] &= \P^{\pi}_{\tau_h}[\c_{h} + c_h \leq B \mid s_h = s, \c_h = \c] \\
        &= \sum_{a'} \pi_h(a' \mid \tau_h) \Pr_{C_h(s,a')}[\c + c \leq B] \\
        &= \Pr_{C_h(s,a)}[\c + c \leq B] \\
        &= 1.
    \end{align*}
    The first line used the law of total probability, the fact that $\P^{\pi}_{\tau_h}[s_h = s, \c_h = \c] = 1$, and the recursive decomposition of cumulative costs. The second line uses law of total probability on $c_h$. The third line follows since $\pi_h(a \mid \tau_h) = 1$ since $\pi_h(\tau_h) = a$ deterministically.
    The last line used the fact that $\Pr_{c \sim C_h(s,a)}[\c + c \leq B] = 1$. Thus, we see that $(s', \c') \in SE_{h+1}$.
\end{proof}

\begin{claim}\label{claim: SE-in-mS}
    For all $h \in [H+1]$, $\mS_h \supseteq SE_h$.
\end{claim}

\begin{proof}
    We proceed by induction on $h$. For the base case, we consider $h = 1$. Observe that for any $\pi \in \Pi$, the only $\tau_1$ that has non-zero probability is $\tau_1 = s_0$ since $M$ starts at time $1$ in state $s_0$. Also, $\c_1 = 0$ since no cost has been accrued by time $1$. Thus, $SE_1 \subseteq \set{(s_0,0)} = \mS_1$. 

    For the inductive step, we consider any $h \geq 1$. Let $(s',\c') \in SE_{h+1}$. By definition, there exists some $\pi \in \Pi$ and some $\tau \in \H$ satisfying,
    \[\P^{\pi}_{\tau_{h+1}}[s_{h+1} = s', \c_{h+1} = \c'] = 1 \text{ and } \forall k \in [h], \; \P^{\pi}_{\tau_k}[\c_{k+1} \leq B] = 1.\]
    Decompose $\tau_{h+1} = (\tau_h, a, c, s')$ where $s_h = s$ and $\c_h = \c$. Since $\P^{\pi}_{\tau_{h+1}}[s_{h+1} = s', \c_{h+1} = \c'] = 1$, we observe that 
    \[0 < \P^{\pi}[\tau_{h+1}] = \P^{\pi}[\tau_h] \pi_h(a \mid \tau_h) C_h(c \mid s,a) P_h(s' \mid s,a).\] 
    Thus, $\P^{\pi}_{\tau_h}[s_h = s, \c_h = \c] = 1$. Also, we immediately know that $\P^{\pi}_{\tau_k}[\c_{k+1} \leq B]$ $\forall k \in [h-1]$ since any sub-history of $\tau_h$ is also a sub-history of $\tau_{h+1}$. Hence, $(s,\c) \in SE_h$ and so the induction hypothesis implies that $(s,\c) \in \mS_h$. We have already seen that $\c' = \c + c$, $C_h(c \mid s,a) > 0$ and $P_h(s' \mid s,a)$. To show that $(s',\c') \in \mS_{h+1}$, it then suffices to argue that $\Pr_{c \sim C_h(s,a)}[\c + c \leq B] = 1$. To this end, observe as in \cref{claim: mS-in-SE} that,
    \[1 = \P^{\pi}_{\tau_h}[\c_{h+1} \leq B] = \sum_{a'} \pi_h(a' \mid \tau_h) \Pr_{C_h(s,a')}[\c + c \leq B].\]
    This implies that for all $a'$, $\Pr_{C_h(s,a')}[\c + c \leq B] = 1$, otherwise we would have,
    \[\sum_{a'} \pi_h(a' \mid \tau_h) \Pr_{C_h(s,a')}[\c + c \leq B] < \sum_{a'} \pi_h(a' \mid \tau_h) = 1,\]
    which is a contradiction. Thus, $c + C_h(s,a) \leq B$ and so $(s',\c') \in \mS_{h+1}$.
\end{proof}

\paragraph{Proof of \cref{lem: feasibility}.}

\begin{proof}\
    \paragraph{First Claim.} 
    Fix any $(s,\c)$ and suppose that $\P^{\pi}_M[s_h = s, \c_h = \c] > 0$ where $\pi \in \Pi_M$. Since $\pi \in \Pi_M$, we know that $\P^{\pi}_M\brac{\forall k \in [H] \sum_{t = 1}^k c_t \leq B} = 1$. Since the history distribution has finite support whenever the cost distributions do, we see for any history $\tau_{h+1} \in \H_{h+1}$ with $\P^{\pi}_M[\tau_{h+1}] > 0$ it must be the case that $\c_{h+1} \leq B$. 
    In fact, it must also be the case that $\Pr_{c \sim C_h(s_h,a_h)}[c + \c_h \leq B] = 1$ otherwise there exists a realization of $c$ and $\c_h$ under $\pi$ for which the anytime constraint would be violated. 
    
    Moreover, this must hold for any subhistory of $\tau_{h+1}$ since those are also realized with non-zero probability under $\pi$. In symbols, we see that $\P^{\pi}_{\tau_k}[\c_{k+1} \leq B] = 1$ for all $k \in [h]$. Thus, $(s,\c) \in SE_h$. Since $(s,\c)$ was arbitrary, we conclude by \cref{lem: safe-exploration} that $\mS_h = SE_h \supseteq \mathcal{F}_h$. 

    observe that $|\mS_1| = 1$. By the inductive definition of $\mS_{h+1}$, we see that for any $(s,\c) \in \mS_h$, $(s,\c)$ is responsible for adding at most $S\sum_{a \in \A}|C_h(s,a)| \leq SAn$ pairs $(s',\c')$ into $\mS_{h+1}$. Specifically, each next state $s'$, current action $a$, and current cost $c \in C_h(s,a)$ yields at most one new element of $\mS_{h+1}$. Thus, $|\mS_{h+1}| \leq SAn |\mS_h|$. Since $S,A,n < \infty$, we see inductively that $|\mS_h| < \infty$ for all $h \in [H+1]$.

    \paragraph{Second Claim.} Suppose that for each $(h,s)$ there is some $a$ with $C_h(s,a) = \{0\}$. For any $(s,\c) \in SE_{h+1}$, we know that there exists some $\pi$ and $\tau_{h+1}$ for which $\P^{\pi}_{\tau_k}[\c_{k+1} \leq B] = 1$ for all $k \in [h]$. Now define the deterministic policy $\pi'$ by $\pi'_k(\tau_k) = \pi_k(\tau_k)$ for all subhistories $\tau_k$ of $\tau_{h+1}$, and $\pi'_k(\tau_k) = a$ for any $a$ with $C_k(s_k,a) = \{0\}$ otherwise. Clearly, $\pi'$ never accumulates more cost than a subhistory of $\tau_{h+1}$ since it always takes $0$ cost actions after inducing a different history than one contained in $\tau_{h+1}$. 
    
    Since under any subhistory of $\tau_{h+1}$, $\pi'$ satisfies the constraints by definition of $SE_{h+1}$, we know that $\pi' \in \Pi_M$. We also see that $\P^{\pi'}_M[s_{h+1} = s, \c_{h+1} = \c] > 0$ and so $(s, \c) \in \mathcal{F}_{h+1}$. Since $(s,\c)$ was arbitrary we have that $SE_{h+1} = \mathcal{F}_{h+1}$. As $h$ was arbitrary the claim holds.
\end{proof}

\begin{observation}\label{obs: bounded}
    For all $h > 1$, if $(s,\c) \in \mS_h$, then $\c \leq B$. 
\end{observation}

\begin{proof}
    For any $h \geq 1$, any $(s,\c) \in \mS_{h+1}$ satisfies $\c' = \c + c$ where $c \in C_h(s,a)$ and $\Pr_{c \sim C_h(s,a)}[\c + c \leq B] = 1$. Since $C_h(s,a)$ has finite support, this means that for any such $c \in C_h(s,a)$, we have that $\c+c \leq B$. In particular, $\c' = \c + c \leq B$.
\end{proof}

\subsection{Proof of \texorpdfstring{\cref{lem: value}}{subsec: Lemma 3}}

The tabular \emph{policy evaluation equations} (Equation 4.2.6 \cite{MDP-book}) naturally translate to the cost setting as follows:
\begin{equation*}
    V_h^{\pi}(\tau_h) = \sum_{a \in \A} \pi_h(a \mid \tau_h) \paren{r_h(s,a) + \sum_{c} \sum_{s'} C_h(c \mid s,a) P_h(s' \mid s,a) V_{h+1}^{\pi}(\tau_{h}, a, c, s')}.
\end{equation*}
We can write this more generally in the compact form:
\begin{equation}\label{equ: PE}\tag{PE}
    V_h^{\pi}(\tau_h) = \E^{\pi}_{\tau_h}\brac{r_h(s,a) + \E^{\pi}_{\tau_{h+1}}\brac{V_{h+1}^{\pi}(\tau_{h+1})}}.
\end{equation}
The classic \emph{Bellman Optimality Equations} (Equation 4.3.2 \cite{MDP-book}) are,
\begin{equation*}
    \mV^*_h(s) = \max_{a \in \A(s)} r_h(s,a) + \E_{s'} \brac{\mV^*_{h+1}(s')}
\end{equation*}
Observe that the optimality equations for $\mM$ are,
\[\mV_h^*(\ms) = \max_{\ma \in \mA_h(\ms)} \mr_h(\ms,\ma) + \E_{\ms'}\brac{ \mV_{h+1}^*(\ms')},\]
which reduce to
\begin{equation}\label{equ: be}\tag{BE}
    \mV_h^*(s,\c) = \max_{a: \Pr_{c \sim C_h(s,a)}[\c + c \leq B] = 1} r_h(s,a) + \E_{c,s'}\brac{\mV_{h+1}^*(s', \c+c)},
\end{equation}
where $\mV^*_{H+1}(s,\c) = 0$. We then define $\mV^*_h(s,\c) = \sup_{\pi} V_h^{\pi}(s,\c)$. Note, if $\pi$ chooses any action $a$ for which $a \not \in \mA(s,\c)$, then $V_h^{\pi}(s,\c) := -\infty$ and we call $\pi$ infeasible for $\mM$.

\begin{observation}\label{lem: traj-update}
    For any $\tau_h \in W_h(s, \c)$, if $a \in \A$ satisfies $\Pr_{c \sim C_h(s,a)}[\c + c \leq B] = 1$, $s' \in \S$ satisfies $P_h(s' \mid s,a) > 0$, and $c \in C_h(s,a)$, then for $\tau_{h+1} := (\tau_h, a, c, s')$, $\tau_{h+1} \in W_{h+1}(s', \c + c)$.
\end{observation}

\begin{proof}
    If $\tau_h \in W_h(s,\c)$, then there exists some $\pi \in \Pi$ with,
    \[\P^{\pi}_{\tau_h}\brac{s_h = s, \c_h = \c} = 1 \text{ and } \forall k \in [h-1], \; \P^{\pi}_{\tau_k}\brac{\c_{k+1} \leq B} = 1.\]
    Define $\pi_h(\tau_h) = a$. Immediately, we see, 
    \[\P^{\pi}[\tau_{h+1}] = \P^{\pi}[\tau_h]\pi_h(a \mid \tau_h) C_h(c \mid s,a) P_h(s' \mid s,a) > 0,\]
    so $\P^{\pi}_{\tau_{h+1}}[s_{h+1} = s', \c_{h+1} = \c + c] = 1$. Also, $\P^{\pi}_{\tau_h}[\c_{h+1} \leq B] = \Pr_{c \sim C_h(s,a)}[\c + c \leq B] = 1$ using the same argument as in \cref{claim: mS-in-SE}. Thus, $\tau_{h+1} \in W_{h+1}(s',\c + c)$.
\end{proof}

Now, we split the proof of \cref{lem: value} into two claims. 

\begin{claim}\label{claim: V*>VB}
    For any $h \in [H+1]$, if $(s,\c) \in \mS_h$ and $\tau_h \in W_h(s,\c)$, then $\mV^*_h(s,\c) \geq V^*(\tau_h).$
\end{claim}

\begin{proof}
    We proceed by induction on $h$. For the base case, we consider $h = H+1$. Let $(s,\c) \in \mS_{H+1}$ and $\tau_{H+1} \in W_{H+1}(s, \c)$. By definition, $\mV_{H+1}^*(s, \c) = 0$. If $\Pi_M(\tau_{H+1}) = \varnothing$, then $V^*(\tau_{H+1}) = -\infty < 0 = \mV_{H+1}^*(s, \c)$. Otherwise, for any $\pi \in \Pi_M(\tau_{H+1})$, $V_{H+1}^{\pi}(\tau_{H+1}) = 0$ by definition. Since $\pi$ was arbitrary we see that $V^*(\tau_{H+1}) = 0 \leq 0 = \mV_{H+1}^*(s, \c)$. 

    For the inductive step, we consider $h \leq H$. Fix any $(s,\c) \in \mS_h$ and let $\tau_h \in W_h(s, \c)$. If $\Pi_M(\tau_h) = \varnothing$, then $V^*(\tau_{h}) = -\infty \leq \mV_{h}^*(s, \c)$.
    Otherwise, fix any $\pi \in \Pi_M(\tau_h)$. 
    Suppose $\tau_{h+1} = (\tau_h, a, c, s')$ where $\pi_h(a \mid \tau_h) > 0$, $C_h(c \mid s,a) > 0$, and $P_h(s' \mid s,a) > 0$.
    For any full history $\tau \in \H$ satisfying $\P^{\pi}_{\tau_{h+1}}[\tau] > 0$, we have $\P^{\pi}_{\tau_h}[\tau] = \P^{\pi}_{\tau_{h+1}}[\tau]\P^{\pi}_{\tau_h}[\tau_{h+1}] > 0.$
    Since $\pi \in \Pi_M(\tau_h)$, we know that for all complete histories $\tau \in \H$ with $\P^{\pi}_{\tau_h}[\tau] > 0$ that  $\c_{k+1} \leq B$ for all $k \in [H]$. 
    Consequently, for any $\tau \in \H$ satisfying $\P^{\pi}_{\tau_{h+1}}[\tau] > 0$, $\c_{k+1} \leq B$ for all $k \in [H]$. This means that $\P^{\pi}_{\tau_{h+1}}[\c_{k+1} \leq B] = 1$ for all $k \in [H]$ and so $\pi \in \Pi_M(\tau_{h+1})$. 

    By \eqref{equ: be}, 
    \begin{align*}
        \mV_h^*(s,\c) &= \max_{a: \Pr_{c \sim C_h(s,a)}[\c + c \leq B] = 1} r_h(s,a) + \E_{c,s'}\brac{\mV_{h+1}^*(s', \c+c)} \\
        &\geq \max_{a: \Pr_{c \sim C_h(s,a)}[\c + c \leq B] = 1} r_h(s,a) + \E_{c, s'}\brac{V_{h+1}^{*}(\tau_{h+1})} \\
        &\geq \sum_{a: \Pr_{c \sim C_h(s,a)}[\c + c \leq B] = 1} \pi_h(a \mid \tau_h) \paren{r_h(s,a) + \E_{\tau_{h+1}}\brac{V_{h+1}^{*}(\tau_{h+1})}} \\
        &\geq \sum_{a: \Pr_{c \sim C_h(s,a)}[\c + c \leq B] = 1} \pi_h(a \mid \tau_h) \paren{r_h(s,a) + \E_{\tau_{h+1}}\brac{V_{h+1}^{\pi}(\tau_{h+1})}} \\
        &= \sum_{a \in \A} \pi_h(a \mid \tau_h) \paren{r_h(s,a) + \E_{\tau_{h+1}}\brac{V_{h+1}^{\pi}(\tau_{h+1})}} \\
        &= V_h^{\pi}(\tau_h).
    \end{align*}
    The second line follows from the induction hypothesis, where $\tau_{h+1} = (\tau_h,a,c,s') \in W_{h+1}(s',\c + c)$ by \cref{lem: traj-update}. The third line follows since the finite maximum is larger than the finite average (see Lemma 4.3.1 of ~\cite{MDP-book}). The fourth line follows since $\pi \in \Pi_M(\tau_{h+1})$ implies $V^*(\tau_{h+1}) \geq V_{h+1}^{\pi}(\tau_h)$.
    $\Pr_{c \sim C_h(s,a)}[\c + c \leq B] = 1$. The fifth line follows since $\pi$ cannot place non-zero weight on any action $a$ satisfying $\c + C_h(s,a) > B$. Otherwise, we would have, 
    \[\P^{\pi}_{\tau_h}[\c_{h+1} > B] \geq \pi_h(a \mid \tau_h)\Pr_{c \sim C_h(s,a)}[\c + c > B] > 0,\] 
    contradicting that $\pi \in \Pi_M(\tau_h)$.
    The final line uses \eqref{equ: PE}.  

    Since $\pi$ was arbitrary, we see that $\mV_h^*(s,\c) \geq V^*(\tau_h)$. 
    
\end{proof}

\begin{claim}\label{claim: V*<VB}
    For any $h \in [H+1]$, if $(s,\c) \in \mS_h$ and $\tau_h \in W_h(s,\c)$, then $\mV^*_h(s,\c) \leq V^*(\tau_h).$
\end{claim}

\begin{proof}
    We proceed by induction on $h$. For the base case, we consider $h = H+1$. If $(s,\c) \in \mS_{H+1}$ and $\tau_{H+1} \in W_{H+1}(s,\c)$, then by definition there exists some $\pi \in \Pi$ for which $\P^{\pi}[\tau_{H+1}] > 0$ and $\P^{\pi}_{\tau_k}[\c_{k+1} \leq B] = 1$ for all $k \in [H]$. We saw in the proof of \cref{claim: SE-in-mS} that for any $k \leq H$, $(s_{k+1}, \c_{k+1}) \in \mS_{k+1}$. Thus, by \cref{obs: bounded}, we have $\c_{k+1} \leq B$ for all $k \in [H]$ which implies that $\P^{\pi}_{\tau_{H+1}}[\c_{k+1} \leq B] = 1$ for all $k \in [H]$. Hence, $\pi \in \Pi_M(\tau_{H+1})$ is a feasible solution to the optimization defining $V^*(\tau_{H+1})$ implying that $V^*(\tau_{H+1}) = 0$. By definition, we also have $\mV^*_{H+1}(s,\c) = 0 \leq V^*(\tau_{H+1})$.

    For the inductive step, we consider $h \leq H$. Let $(s,\c) \in \mS_h$ and $\tau_{h} \in W_h(s,\c)$. 
    Consider the deterministic optimal partial policy $\bar \pi$ for $\mM$ defined by solutions to \eqref{equ: be}. Formally, for all $t \geq h$,
    \[\bar \pi_t(s, \c) \in \argmax_{a: \Pr_{C_t(s,a)} [\c + c \leq B] = 1} r_t(s,a) + \E_{c,s'}\brac{\mV_{t+1}^*(s', \c+c)}.\]
    If there is no feasible action for any of these equations of form $(t,s',\c')$ where $t \geq h$ and $(s',\c') \in \mS_t$ are reachable from $(h,s,\c)$ with non-zero probability, then $\mV^*_h(s,\c) = -\infty$. In this case, clearly $\mV^*_h(s, \c) \leq V^*(\tau_h)$. Otherwise, suppose solutions to \eqref{equ: be} exist so that $\bar \pi$ is well-defined from $(s,\c)$ at time $h$ onward.
    It is well known (see Theorem 4.3.3 from ~\cite{MDP-book}) that $\mV^{\bar \pi}_t(s',\c') = \mV^*_t(s',\c')$ for all $t \geq h$ and all $(s',\c') \in \mS_t$. We unpack $\bar \pi$ into a partial policy $\pi$ for $M$ defined by,
    \begin{align*}
        \pi_t(\tau_t) = \begin{cases}
            \bar \pi_t(s_t, \c_t) & \text{ if } (s_t, \c_t) \in \mS_t \\
            a_1 & \text{ o.w. }
        \end{cases}
    \end{align*}
    Here, $a_1$ is an arbitrary element of $\A$. To make $\pi$ a full policy, we can define $\pi_t$ arbitrarily for any $t < h$.

    We first show that for all $t \geq h$, $\P^{\pi}_{\tau_h}[(s_t,\c_t) \in \mS_t] = 1$. We proceed by induction on $t$. For the base case, we consider $t = h$. By assumption, $\tau_h \in W_h(s,\c)$ so $(s_h,\c_h) = (s,\c) \in \mS_h$.  Thus, $\P^{\pi}_{\tau_h}[(s_h,\c_h) \in \mS_h] = \P^{\pi}_{\tau_h}[(s,\c) \in \mS_h] = 1$.

    For the inductive step, we consider $t \geq h$. By the induction hypothesis, we know that $\P^{\pi}_{\tau_h}[(s_t,\c_t) \in \mS_t] = 1$. By the law of total probability, it is then clear that,
    \begin{align*}
        &\P^{\pi}_{\tau_h}[(s_{t+1},\c_{t+1}) \in \mS_{t+1}] = \P^{\pi}_{\tau_h}[(s_{t+1},\c_{t+1}) \in \mS_{t+1} \mid (s_t,\c_{t}) \in \mS_t] \\
        &= \sum_{(s',\c') \in \mS_t} \P^{\pi}_{\tau_h}[(s_{t+1},\c_{t+1}) \in \mS_{t+1} \mid s_t = s', \c_t = \c']\P^{\pi}_{\tau_h}[s_t = s', \c_t = \c'].
    \end{align*}
    Above we have used the fact that for any $(s',\c') \in \mS_t$, the event that $\{s_t = s', \c_t = \c', (s_t,\c_t) \in \mS_t\} = \{s_t = s', \c_t = \c'\}$.
    
    For any $\tau_t$ with $(s_t,\c_t) = (s',\c') \in \mS_t$, by definition, $\pi_t(\tau_t) = \bar \pi_t(s',\c') = a' \in \{a \in \A \mid \c' + C_t(s',a) \leq B\}$. By the inductive definition of $\mS_{t+1}$, we then see that $(s'',\c' + c') \in \mS_{t+1}$ for any $s'' \sim P_t(s',a')$ and $c' \sim C_t(s',a')$. Hence, $\P^{\pi}_{\tau_h}[(s_{t+1},\c_{t+1}) \in \mS_{t+1} \mid s_t = s', \c_t = \c'] = 1$. We then see that,
    \begin{align*}
        \P^{\pi}_{\tau_h}[(s_{t+1},\c_{t+1}) \in \mS_{t+1}] &= \sum_{(s',\c') \in \mS_t} \P^{\pi}_{\tau_h}[s_t = s', \c_t = \c'] \\
        &= \P^{\pi}_{\tau_h}[(s_t,\c_t) \in \mS_t] \\
        & = 1
    \end{align*}
    This completes the induction. 
    
    Since under $\tau_h$, $\pi$ induces only histories whose state and cumulative cost are in $\mS$, we see that $\pi$'s behavior is identical to $\bar \pi$'s almost surely. In particular, it is easy to verify by induction using \eqref{equ: PE} and \cref{lem: traj-update} that,
    \begin{align*}
        V_h^{\pi}(\tau_h) &= \E_{\tau_h}^{\pi}\brac{r_h(s,a) + E_{\tau_{h+1}}\brac{V_{h+1}^{\pi}(\tau_{h+1})}} \\
        &= \E_{(s,\c)}^{\bar \pi} \brac{r_h((s,\c),a) + E_{(s',\c')}\brac{\mV_{h+1}^{\bar \pi}(s',\c')}} \\
        &= V_h^{\bar \pi}(s,\c) \\
        &= \mV_h^*(s,\c).
    \end{align*}
    By \cref{obs: bounded}, we see if $(s_{k+1},\c_{k+1}) \in \mS_{k+1}$ then $\c_{k+1} \leq B$. It is then clear by monotonicity of probability that $\P^{\pi}_{\tau_h}[c_{k+1} \leq B] \geq \P^{\pi}_{\tau_h}[(s_{k+1},c_{k+1}) \leq \mS_{k+1}] = 1$ for all $k \in [H]$. Hence, $\pi \in \Pi_M(\tau_h)$ and so $\mV_h^*(s,\c) = V_h^{\pi}(\tau_h) \leq V^*(\tau_h)$.
    
\end{proof}

\begin{observation}[Cost-Augmented Probability Measures]\label{obs: augmented-measure}
    We note we can treat $\bar \pi$ defined in the proof of \cref{claim: V*<VB} as a history dependent policy in the same way we defined $\pi$. Doing this induces a probability measure over histories. We observe that measure is identical as the one induced by the true history-dependent policy $\pi$. Thus, we can directly use augmented policies with $M$ and reason about their values and costs with respect to $M$.
\end{observation}

\subsection{Proof of \texorpdfstring{\cref{thm: optimality}}{subsec: theorem 3}}

\begin{proof}
    From \cref{lem: value}, we see that $V^* = V^*(s_0) = \mV_1^*(s_0,0) = \mV^*$. Furthermore, in \cref{claim: V*<VB}, we saw the policy defined by the optimality equations \eqref{equ: be} achieves the optimal value, $\mV^{\bar \pi} = \mV^* = V^*$. Furthermore, $\bar \pi$ behaves identically to a feasible history-dependent policy $\pi$ almost surely. In particular, as argued in \cref{claim: V*<VB} both policies only induce cumulative costs appearing in $\mS_h$ at any time $h$ and so by \cref{obs: bounded} we know that both policies' cumulative costs are at most $B$ anytime.
\end{proof}

\subsection{Proof of \texorpdfstring{\cref{cor: reduction}}{subsec: corollary 2}}

The theorem follows immediately from \cref{thm: optimality} and the argument from the main text.

\subsection{Proof of \texorpdfstring{\cref{prop: complexity}}{subsec: proposition 3}}

\begin{proof}
By definition of $\mS$, it is clear that $|\mS_h| \leq |\S|\D$, and by inspection, we see that $|\mA| \leq |\A|$. The agent can construct $\mS$ using our forward induction procedure, \cref{alg: cost-states}, in $O(\sum_{h = 1}^{H-1} |\mS_{h}|SAn) = O(HS^2 An \D)$ time. Also, the agent can compute $\mP$ by forward induction in the same amount of time so long as the agent only records the non-zero transitions. Thus, $\mM$ can be computed in $O(HS^2 An \D)$ time. 

\begin{enumerate}
    \item By directly using backward induction on $\mM$ ~\cite{MDP-book}, we see that an optimal policy can be computed in $O(H|\mS|^2|\mA|) = O(HS^2A\D^2)$ time. However, this analysis can be refined: for any sub-problem of the backward induction $(h,(s,\c))$ and any action $a$, there are at most $nS$ state-cost pairs that can be reached in the next period (namely, those of the form $(s',\c + c)$) rather than $S\D$. Thus, backward induction runs in $O(HS^2 An \D)$ time, and so planning in total can be performed in $O(HS^2 An \D)$ time. 
    \item Similarly, PAC (probably-approximately correct) learning can be done with sample complexity $\tilde O(H^3|\mS||\mA|\log(\frac{1}{\delta})/\gamma^2) = \tilde O(H^3SA\D \log(\frac{1}{\delta})/\gamma^2)$~\cite{MDP-PAC}, where $\delta$ is the confidence and $\gamma$ is the accuracy. Note, we are translating the guarantee to the non-stationary state set setting which is why the $|\mS|$ term becomes $S\D$ instead of $HS\D$.
\end{enumerate}

\end{proof}

\subsection{Proof of \texorpdfstring{\cref{lem: precision}}{subsec: lemma 4}}

\begin{proof}
    Suppose each cost is represented with $k$ bits of precision. For simplicity, we assume that $k$ includes a possible sign bit. By ignoring insignificant digits, we can write each number in the form $2^{-i} b_{-i} + \ldots 2^{-1} b_{-1} + 2^0 b_{0} + \ldots + 2^{k - i-1}b_{k - i}$ for some $i$. By dividing by $2^{-i}$, each number is of the form $2^0 b_{0} + \ldots + 2^{k-1} b_{k-1}$. Notice, the largest possible number that can be represented in this form is $\sum_{i = 0}^{k-1} 2^i = 2^{k} - 1$. Since at each time $h$, we potentially add the maximum cost, the largest cumulative cost ever achieved is at most $2^k H - 1$. Since that is the largest cost achievable, no more than $2^k H$ can ever be achieved through all $H$ times. Similarly, no cost can be achieved smaller than $-2^k H$. 
    
    Thus each cumulative cost is in the range $[-2^kH+1,2^kH-1]$ and so at most $2^{k+1}H$ cumulative costs can ever be created. By multiplying back the $2^{-i}$ term, we see at most $2^{k+1} H$ costs are ever generated by numbers with $k$ bits of precision. Since this argument holds for each constraint independently, the total number of cumulative cost vectors that could ever be achieved is $(H2^{k+1})^d$.
    Hence, $\mathcal{D} \leq H^d2^{(k+1)d}$.
\end{proof}

\subsection{Proof of \texorpdfstring{\cref{thm: fpt}}{subsec: theorem 4}}
\cref{thm: fpt} follows immediately from \cref{prop: complexity}, \cref{lem: precision}, and the definition of fixed-parameter tractability~\cite{FPT}.

\section{Proofs for \texorpdfstring{\cref{sec: approximation}}{sec: approximation}}

For any $h$ we let $\hat{c}_{h+1} := f(\tau_{h+1})$ be a random variable of the history defined inductively by $\hat{c}_1 = 0$ and $\hat{c}_{k+1} = f_k(\hat{c}_k,c_k)$ for all $k \leq h$. Notice that since $f$ is a deterministic function, $\hat{c}_{k}$ can be computed from $\tau_{h+1}$ for all $k \in [h+1]$. Then, a probability distribution over $\hat{c}$ is induced by the one over histories. As such, approximate-cost augmented policies can also be viewed as history-dependent policies for $M$ as in \cref{obs: augmented-measure}.

\subsection{Proof of \texorpdfstring{\cref{lem: approx}}{subsec: lem-approx}}

\begin{proof}
    We proceed by induction on $h$. Fix any feasible policy $\pi$ for $\hat{M}$. For the base case, we consider $h = 1$.
    By definition, $\c_1 = 0 = \hat{c}_1$ and so the claim trivially holds. For the inductive step, we consider any $h \geq 1$. By the induciton hypothesis, we know that $\hat{c}_{h} \leq \c_h \leq \hat{c}_{h} + (h-1)\ell$ or $\hat{c}_{h}, \c_{h} \leq B - (H-h+1)\cmax$ almost surely. 
    We split the proof into cases.

    \begin{enumerate}
        \item First, suppose that $\hat{c}_{h} \leq \c_h \leq \hat{c}_{h} + (h-1)\ell$.
        \begin{enumerate}
            \item Furthermore, suppose that $\hat{c}_{h} + c_{h} \geq B - (H-h)\cmax$ so that $\hat{c}_{h+1} = f_h(\hat{c}_{h}, c_h) = \hat{c}_{h} + \floor{\frac{c_h}{\ell}}\ell$. By definition of the floor function, $\floor{\frac{c_h}{\ell}} \leq \frac{c_h}{\ell}$. Thus,
            \[\hat{c}_{h+1} \leq \hat{c}_{h} + \frac{c_h}{\ell}\ell = \hat{c}_{h} + c_h \leq \c_{h} + c_h = \c_{h+1},\]
            holds almost surely, where we used the inductive hypothesis with our case assumption to infer that $\hat{c}_{h} \leq \c_{h}$ almost surely in the second inequality.
            Also, by definition of the floor function, $\frac{c_h}{\ell} \leq \floor{\frac{c_h}{\ell}} + 1$. We then see that,
            \begin{align*}
                \c_{h+1} = \c_h + \frac{c_h}{\ell}\ell 
                \leq \c_h + (\floor{\frac{c_h}{\ell}}+1)\ell 
                \leq \hat{c}_h + (h-1)\ell + \floor{\frac{c_h}{\ell}}\ell + \ell 
                = \hat{c}_{h+1} + h\ell.
            \end{align*}
            The first inequality used the induction hypothesis with our case assumption and the second used the property of floors. 
            \item Now, suppose that $\hat{c}_{h} + c_{h} < B - (H-h)\cmax$ so that $\hat{c}_{h+1} = f_h(\hat{c}_{h}, c_h) = \floor{\frac{B - (H-h)\cmax}{\ell}}\ell$. 
            \begin{enumerate}
                \item If $\c_{h+1} \leq \hat{c}_{h+1}$, then by definition we have,
                \[\c_{h+1}, \hat{c}_{h+1} \leq \floor{\frac{B - (H-h)\cmax}{\ell}}\ell \leq B - (H-h)\cmax,\]
                and we are done.
                \item Otherwise, if $\hat{c}_{h+1} \leq \c_{h+1}$, then we see that,
                \begin{align*}
                    \c_{h+1} &= \c_h + c_h \leq \hat{c}_{h} + (h-1)\ell + c_h < B-(H-h)\cmax + (h-1)\ell \\
                    &\leq (\floor{\frac{B-(H-h)\cmax}{\ell}}+1)\ell + (h-1)\ell = \hat{c}_{h+1} + h\ell,
                \end{align*}
                where the first inequality used the induction hypothesis with our case assumption.
            \end{enumerate}
             
        \end{enumerate}
        \item Lastly, suppose that $\c_{h},\hat{c}_{h} \leq B - (H-h+1)\cmax$. Then, it is clear that,
        \[\c_{h+1} = \c_h + c_h \leq \c_h + \cmax  \leq B - (H-h+1)\cmax + \cmax = B - (H-h)\cmax.\]
        Similarly, we see that either,
        \[\hat{c}_{h+1} = \hat{c}_h + \floor{\frac{c_h}{\ell}}\ell \leq \hat{c}_h + c_h \leq \hat{c}_h + \cmax  \leq B - (H-h+1)\cmax + \cmax = B - (H-h)\cmax,\]
        or,
        \[\hat{c}_{h+1} = \floor{\frac{B - (H-h)\cmax}{\ell}}\ell \leq B - (H-h)\cmax.\]
    \end{enumerate}
    This completes the induction.

    We next show the second claim. By definition, any approximate cost is an integer multiple of $\ell$ where the integer is in the range $\{\floor{\frac{B - H\cmax}{\ell}}, \ldots, \floor{\frac{B}{\ell}}\}$. The number of elements in this set is exactly,
    \[\floor{\frac{B}{\ell}} - \floor{\frac{B - H\cmax}{\ell}} + 1 \leq \frac{B}{\ell} - (\frac{B - H\cmax}{\ell} - 1) + 1 = \frac{H\cmax}{\ell} + 2.\]
    When there are $d$ constraints, this analysis applies to each separately since we do vector operations component-wise. Thus, the total number of approximate costs is $(\frac{H\norm{\cmax}_{\infty}}{\ell} + 2)^d$.

\end{proof}

\subsection{Proof of \texorpdfstring{\cref{thm: approx}}{subsec: thm-approx}}

\begin{proof}
    We first note that the same argument used to prove \cref{thm: optimality} immediately extends to the approximate MDP and implies that any feasible $\pi$ for $\hM$ satisfies $\P_{\hM}^{\pi}[\forall t \in [H], \; \hat{c}_{t+1} \leq B] = 1$. Also, we note since $\hat{c}$ is a deterministic function of the history, we can view any policy $\pi$ for $\hat{M}$ as a cost-history-dependent policy for $M$ similar to in the proof of \cref{obs: augmented-measure}. Thus, \cref{lem: approx} implies that for any feasible $\pi$ for $\hM$ and any $h \in [H+1]$, $\P^{\pi}_M[\hat{c}_h \leq \c_h \leq \hat{c}_{h} + (h-1)\ell \lor \c_h, \hat{c}_h \leq B-(H-h+1)\cmax] = 1$. Since $\hat{c}_{h+1} \leq B$ a.s., we immediately see that $\P^{\pi}_M[\c_{h+1} \leq B + h\ell] = 1$ for all $h \in [H]$. 
    
    Furthermore, we observe that any feasible policy $\pi$ for the anytime constraint is also feasible for $\hM$ since $\P^{\pi}_M[\hat{c}_h \leq \c_h \lor \c_h, \hat{c}_h \leq B-(H-h+1)\cmax] = 1$ implies that $\P^{\pi}_M[\hat{c}_{h+1} \leq B] = 1$ since $\c_{h+1} \leq B$ almost surely. Since the rewards of $\hM$ only depends on the state and action, we see $\pi$ achieves the same value in both MDPs. Thus, $\hV^* \geq V^*$.
    
    Lastly, \cref{lem: approx} implies that $\D_{\hat{M}} \leq (\frac{H\norm{\cmax}_{\infty}}{\ell} + 2)^d$ which with \cref{prop: complexity} gives the storage complexity.
\end{proof}

\subsection{Proof of \texorpdfstring{\cref{cor: additive}}{subsec: cor-additive}}

\begin{proof}
    The proof is immediate from \cref{thm: approx} and \cref{prop: complexity}. 
\end{proof}

\subsection{Proof of \texorpdfstring{\cref{cor: relative}}{subsec: cor-relative}}

\begin{proof}
    The proof is immediate from \cref{thm: approx} and \cref{prop: complexity}. 
\end{proof}

\subsection{Proof of \texorpdfstring{\cref{cor: improved}}{subsec: cor-improved}}
First observe that if $B < 0$ then the instance is trivially infeasible which can be determined in linear time. Otherwise, the immediate cost (in addition to the cumulative cost) induced by any feasible $\pi$ is always in the range $[0,B]$. Specifically, the larger costs $xB$ can never be accrued since there are no negative costs now to offset them, so we can effectively assume that $\cmax \leq B$. Since the floor of any non-negative number is non-negative, the integer multiples of $\ell$ needed are now in the range $[0,\floor{\cmax/\ell}] \subseteq [0, \floor{B/\ell}]$. Thus, we have $O(\frac{H\cmax}{\epsilon})$ approximate costs for the additive approximation since $\ell = \epsilon/ H$, and $O(\frac{H}{\epsilon})$ approximate costs for the relative approximation since $\ell = \epsilon B / H$. The complexities are reduced accordingly.

\subsection{Proof of \texorpdfstring{\cref{prop: relax-hard}}{subsec: prop-relax-hard}}

\begin{proof}
    Note that computing an optimal-value, $\epsilon$-additive solution for the knapsack problem is equivalent to just solving the knapsack problem when $\epsilon < 1$. In particular, since each weight is integer-valued, if the sum of the weights is at most $B + \epsilon < B + 1$ then it is also at most $B$. By scaling the weights and budget by $\ceil{2\epsilon}$, the same argument implies hardness for $\epsilon \geq 1$.  
    
    For relative approximations, we present a reduction from Partition to the problem of finding an optimal-value, $\epsilon$-relative feasible solution to the knapsack problem with negative weights. Again, we focus on the $\epsilon < 1$ regime but note the proof extends using scaling. Let $X = \set{x_1, \ldots, x_n}$ be the set of positive integers input to the partition problem and $Sum(X) := \sum_{i = 1}^n x_i$. Observe that $Sum(X)/2$ must be an integer else the instance is trivially a ``No'' instance.  Define $v_i = 2x_i$ and $w_i = 2x_i$ for each $i \in [n]$. Also, we define a special item $0$ with $v_0 = -Sum(X)$ and $w_0 = -Sum(X)$. We define the budget to be $B = 1$. We claim that there exists some $Y \subseteq [n]$ with $Sum(Y) = Sum(\overline{Y}) = Sum(X)/2$ if and only if there exists an $I \subseteq [n] \cup \set{0}$ with $\sum_{i \in I} v_i \geq 0$ and $\sum_{i \in I} w_i \leq B(1+\epsilon)$.

    \begin{itemize}
        \item $[\implies]$ if $Y$ is a solution to Partition, then we define $I = Y \cup {0}$. We observe that, 
        \[\sum_{i \in I} v_i = -Sum(X) + 2 \sum_{i \in S} x_i = -Sum(X) + 2Sum(X)/2 = 0.\] 
        Similarly, $\sum_{i \in I} w_i = 0 < 1 \leq B(1+\epsilon)$. Thus, $I$ satisfies the conditions.
        
        \item $[\impliedby]$ if $I$ is an $\epsilon$-relative feasible solution to Knapsack, observe that $I$ must contain $0$. In particular, each $w_i = 2x_i \geq 2 > (1+\epsilon) = B(1+\epsilon)$ and so for approximate feasibility to hold it must be the case that a negative weight was included. Let $Y = I \setminus {0}$. 
        Then, we see that,
        \[0 \leq \sum_{i \in I} v_i = -Sum(X) + 2\sum_{i \in Y} x_i = -Sum(X) + 2 Sum(Y).\]
        Thus, $Sum(Y) \geq Sum(X)/2$. Similarly, 
        \[1+\epsilon \geq \sum_{i \in I} w_i = -Sum(X) + 2Sum(Y).\] 
        Thus, $Sum(Y) \leq Sum(X)/2 + (1+\epsilon)/2 < Sum(X)/2 + 1$ since $\epsilon < 1$. Because $Sum(Y)$ is a sum of positive integers, and $Sum(X)/2$ is a positive integer, it must be the case that $Sum(Y) \leq Sum(X)/2$. Pairing this with $Sum(Y) \geq Sum(X)/2$ implies equality holds. Thus, $Y$ is a solution to Partition.
    \end{itemize}

    Since the transformation can be made in linear time, we see that the reduction is polynomial time. Since Partition is NP-hard, we then see finding an optimal-value, $\epsilon$-relative feasible solution to the knapsack problem with negative weights is NP-hard.

\end{proof}

\subsection{Proof of \texorpdfstring{\cref{prop: feasible-scheme}}{subsec: prop-feasible}}

\begin{proof}
    The proof is immediate from \cref{cor: additive} and \cref{cor: relative}.
\end{proof}

\section{Extensions}

\subsection{Generalized Anytime Constraints}

Consider the constraints of the form,

\begin{equation}
    \P^{\pi}_M\brac{\forall k \in [H], \; \sum_{t = 1}^k c_t \in [L_k, U_k]} = 1.
\end{equation}
All of our exact methods carry over to this more general setting by simply tweaking the safe exploration set. In particular, we define,
\begin{multline}
    \mS_{h+1} := \Big\{(s',\c') \in \S \times \Real^d \mid \exists (s,\c) \in \mS_h, \exists a \in \A, 
    \exists c \in C_h(s,a), \\
    \c' = \c + c, \; \Pr_{c \sim C_h(s,a)}[c + \c \in [L_{h}, U_{h}]] = 1, \; P_h(s' \mid s,a) > 0 \Big\}.
\end{multline}
Similarly, each quantity in the analysis changes to consider the different intervals per time step. The proof is otherwise identical. 

For the approximate methods, the additive results imply the costs are at most $U_k + \epsilon$ anytime, and since the costs are independent of the new restrictions, the complexity guarantees are the same. We could similarly give an approximation concerning the lower bound by using pessimistic costs. For the relative approximation, we now define $\ell$ with respect to $|U^{min}| = \min_k |U_k|$ and all costs should lie below $|U^{min}|$. The guarantees then translate over with $|U^{min}|$ taking the role of $|B|$.

\subsection{General Almost-Sure Constraints}

General almost-sure constraints require that,
\begin{equation}
    \P^{\pi}_M\brac{\sum_{t = 1}^H c_t \leq B} = 1.
\end{equation}

This can be easily captured by the generalized anytime constraints by making $L_k$ smaller than $k \cmin$ and $U_k$ larger than $k\cmax$ for any $k < H$ so that the process is effectively unconstrained until the last time step where $U_H = B$. 

Observe then when applying our relative approximation, $U^{min} = U_H = B$ and so the guarantees translate similarly as to the original anytime constraints. In particular, although $\cmax \leq |B|$, the cumulative cost could be up to $H|B|$. This means the multiples of $\ell$ that need to be considered are in the set $\set{\floor{-xH^2/\epsilon}, \ldots, \floor{xH^2/\epsilon}}^d$. This changes the exact constants considered, but the asymptotic guarantees are the same. We do note however that the improvements in \cref{cor: improved} do not extend to the general almost-sure case. 

On the other hand, the additive approximation results now have $\norm{2H\cmax - B}_{\infty}$ terms instead of $\norm{\cmax}_{\infty}$ terms. The asymptotic bounds then have $\norm{\cmax - B/H}_{\infty}$ terms.

\subsection{Infinite Discounting}

If the rewards and costs are discounted, it is easy to see that \cref{thm: optimality} still holds but the resultant MDP has infinite states and discontinuous reward function. However, our approximate methods work well. By simply using the horizon $H$ to be the earliest time in which $\sum_{t = H+1}^{\infty} \gamma^t c_t \leq \epsilon$ almost surely, we can use our reduction to get an $\epsilon$-additive feasible policy. Pairing this with our approximation algorithms gives a computationally efficient solution. To get a desired accuracy the effective horizon $H$ may need to be increased before using the approximation algorithms.

\section{Additional Experiments}

For all of our experiments, we generate cMDPs built from the same structure as those in the proof of \cref{thm: np-hardness}, but with deterministic costs and rewards chosen uniformly at random over $[0,1]$. Formally, we consider cMDPs with a single state ($\S = \{0\})$, two actions ($\A = \{0,1\}$), and non-stationary cost and reward functions. For all $h$, $r_h(s,0) = c_h(s,0) = 0$, and $r_h(s,1), c_h(s,1) \in U[0,1]$. Despite these being worst-case hard instances, capturing the knapsack problem, we conjecture that under well-concentrated costs, such as uniform costs, that our relative approximate method \cref{cor: relative} and no-violation methods \cref{prop: feasible-scheme} would perform well. We test this and the performance of all of our methods on these hard cMDPs. Since the complexity blowup for cMDPs is in the time horizon, we focus on testing our methods on hard instances with varying time horizons.

\subsection{Approximation Scheme Performance}

We first test how well our approximation scheme \cref{alg: approximation} does compared to our exact, FPT-reduction, method \cref{cor: reduction}. We see even for the fairly large choice of precision parameter, $\epsilon = .1$, that our approximation performs very well. We tested both methods on $N = 5$ hard cMDPs with fixed time horizons in the range of $1, \ldots, H_{\max} = 16$. Note, already for a single instance with $H = 15$, the exact method takes over a minute to run, which is unsurprising as the complexity grows like $2^{15}$. However, we will show our approximate methods easily handle much larger instances in \cref{subsec: scale}.

Here, we consider two extreme choices of budget, $Bs = [.1,10]$. The first is chosen so that few costs are induced before hitting the budget, so we expect even the exact method to be fast. The second is chosen to be close to the maximum time horizon, which is an upper bound on the value ever achievable, so that many cumulative costs could be induced. We expect the exact method would run quite slowly on a large budget.

In all of our experiments, we present the worst-case for our methods. In particular, out of the $N$ trials for the cMPDs we run on a fixed $H$, we report the run-time that was worst for our approximation method, and the approximate value that was closest to the exact method (the approximate value, in general, could be much larger than the exact method due to the optimistic costs \cref{cor: relative}). 

\begin{figure}
    \centering
    \includegraphics[scale=.55]{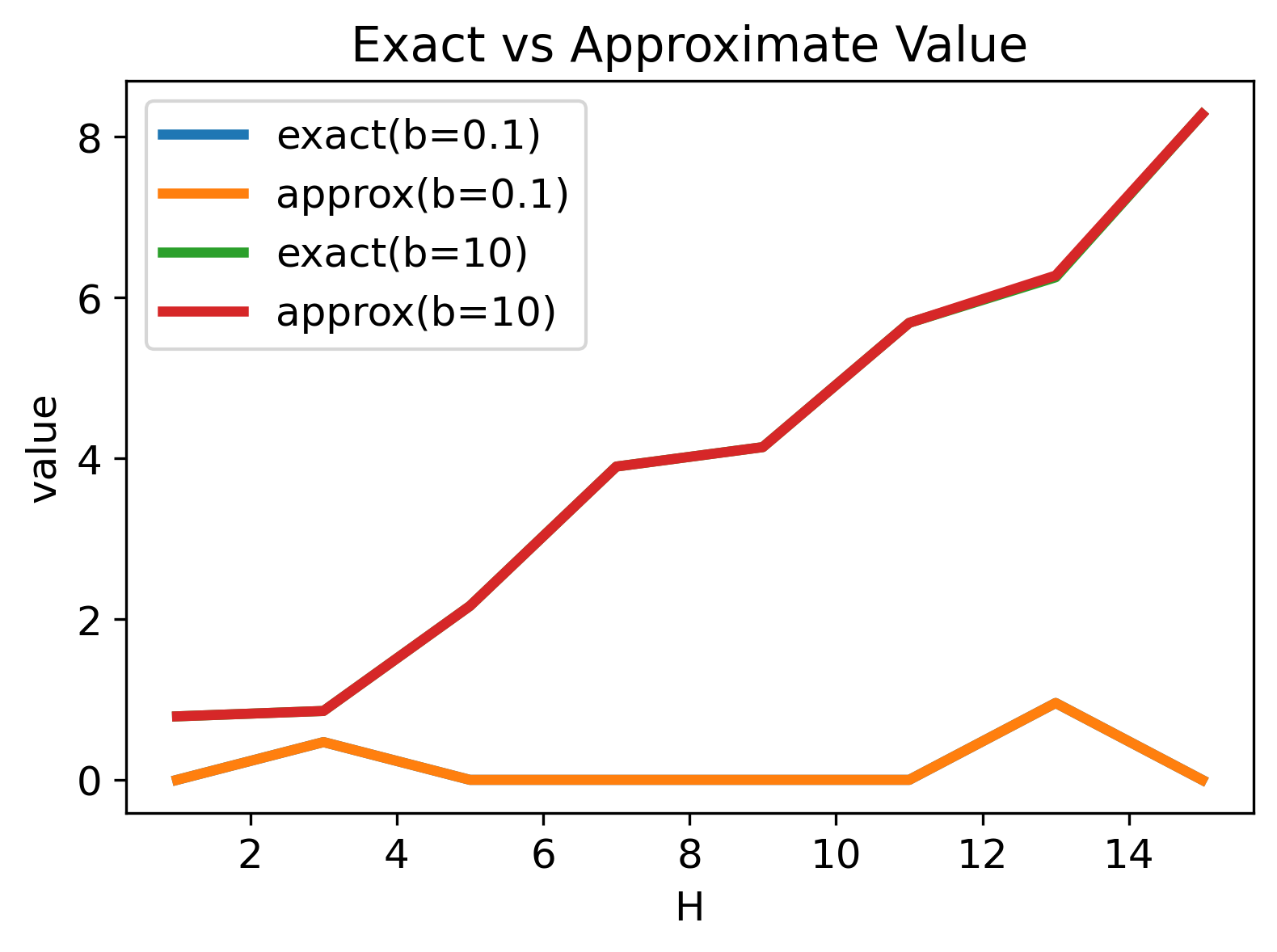}
    \includegraphics[scale=.55]{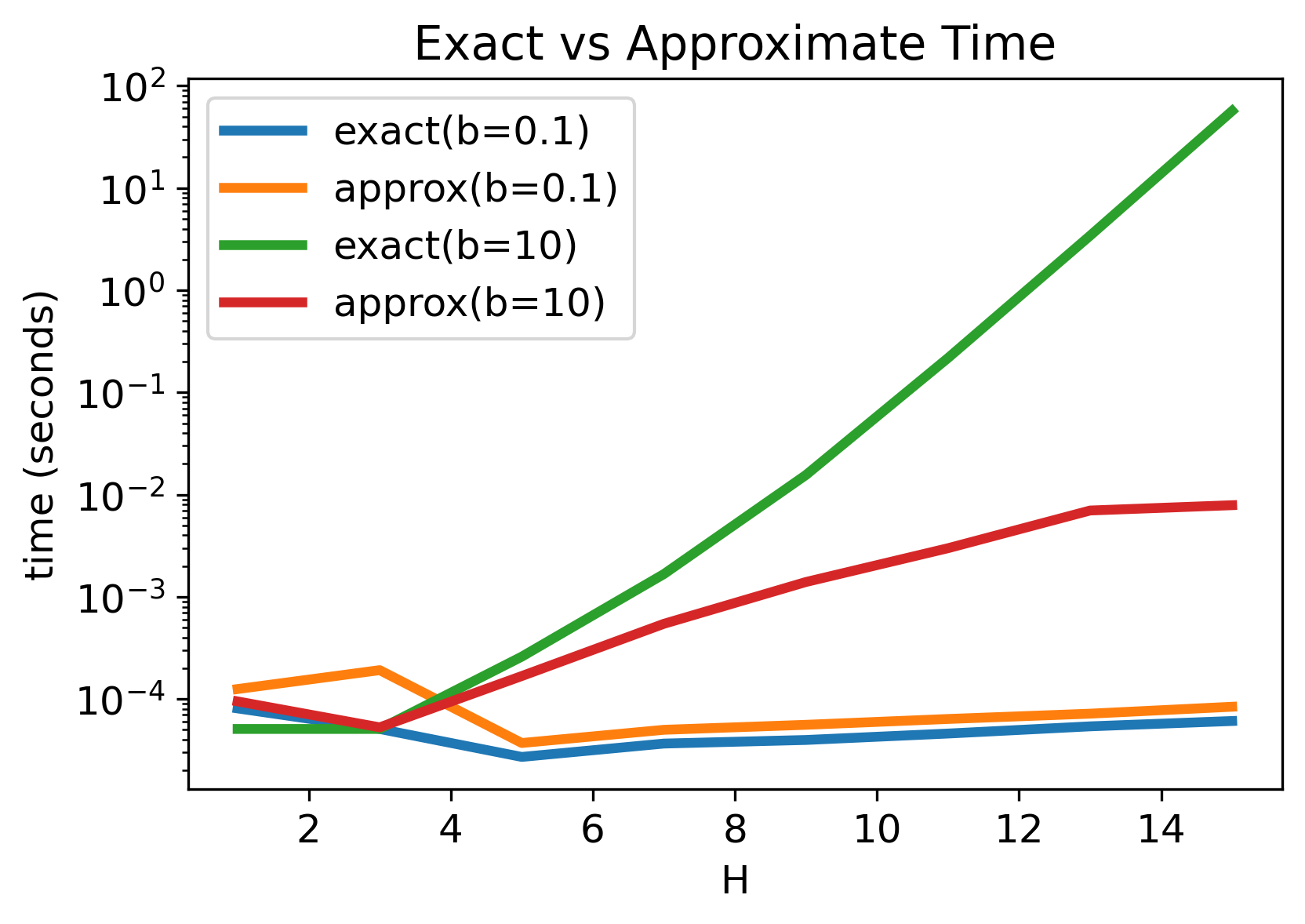}
    \caption{Our Approximation Scheme vs Our Exact Method on hard cMDP instances.}
    \label{fig: approx}
\end{figure}

The results are given in \cref{fig: approx}. We see that for small budgets, the approximate method does not have much benefit and can even be slightly faster. However, for larger budgets, the approximation scheme is leagues faster, completing in one-hundredth of a second instead of over a minute and a half like the exact method. We also see in these instances the approximation method exactly matched the exact method value, which is an indicator that the approximate policies might not go over budget. We see further evidence of this in the no-violation results we see next.

\subsection{No-Violation Scheme Performance}

Next, we use the same setup (except $H_{\max} = 14$ now), but compare the no-violation scheme to the exact method. Here, we report the values for trials when the no-violation scheme was the furthest below the exact method (recall that, unlike the approximation scheme, the no-violation scheme is generally suboptimal). The results are summarized in \cref{fig: no_violation}. We see in fact that the no-violation scheme gets optimal value in nearly every trial we ran! Also, the scheme is much faster than the exact method as expected.

\begin{figure}
    \centering
    \includegraphics[scale = .55]{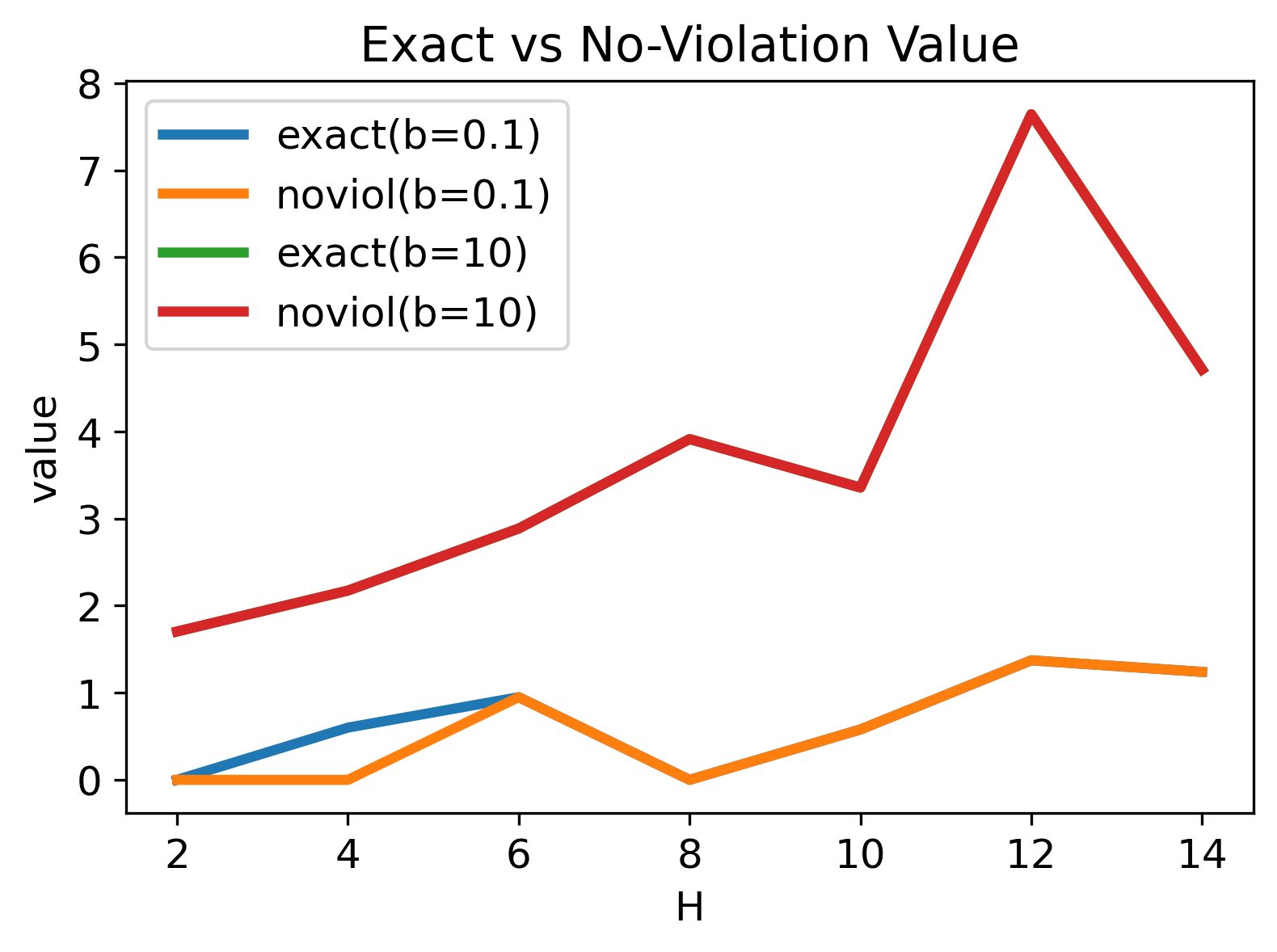}
    \includegraphics[scale=.55]{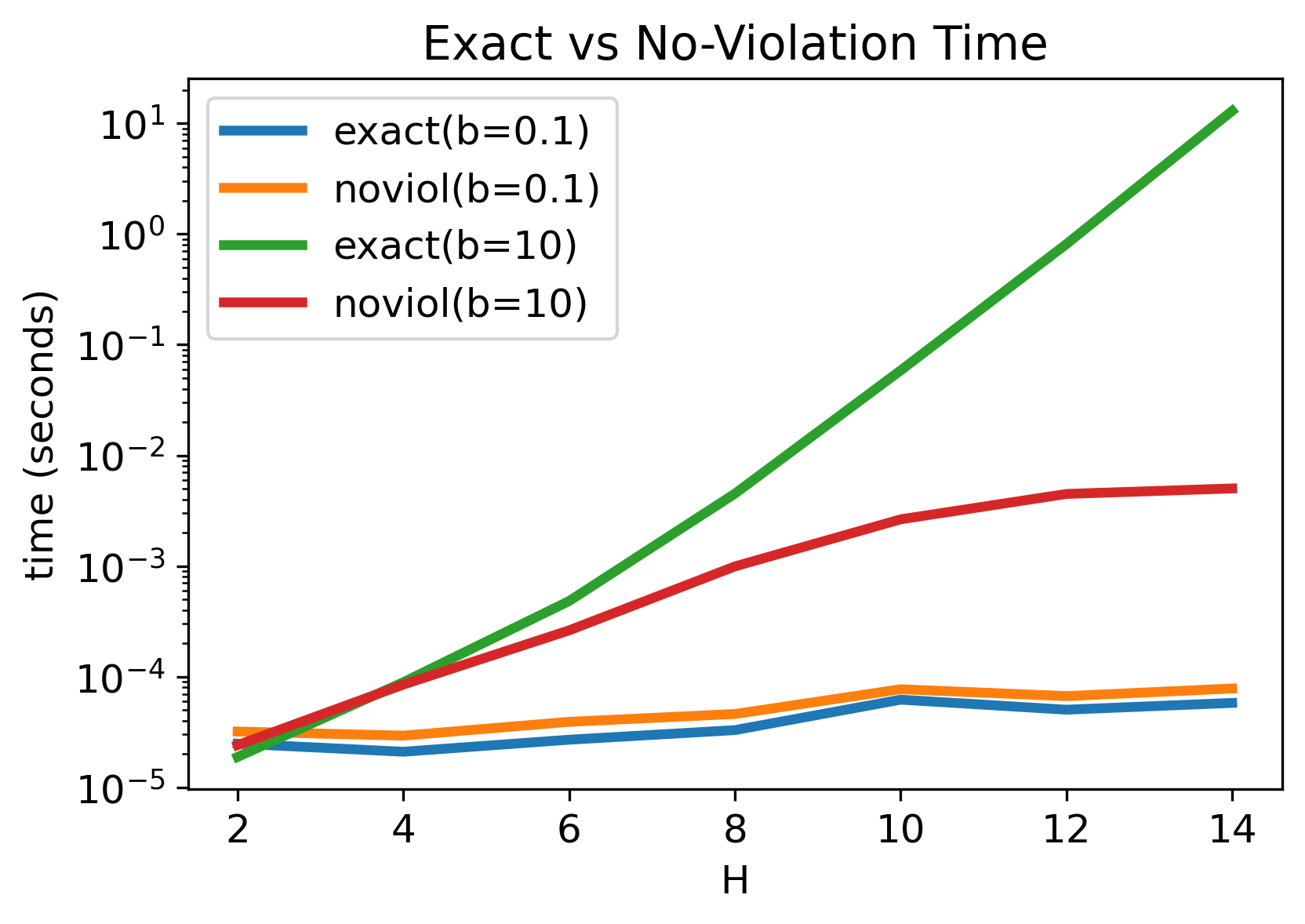}
    \caption{Our No-Violation Scheme vs Our Exact Method on hard cMDP instances.}
    \label{fig: no_violation}
\end{figure}

\subsection{Approximation vs No-Violation Scheme}

Since the exact method is too slow to handle instances with $H = 20$, we need another way to test the efficacy of the approximation and no-violation schemes. In fact, we can just compare them to each other. Since the approximation scheme is always at least the optimal value, and the no-violation scheme is at most the optimal value when they coincide both schemes are optimal. Furthermore, since the no-violation scheme does not violate the budget, if they coincide it gives evidence that the approximate method is producing policies that are within budget anytime.  

Now, we let $H_{\max} = 50$ and perform $N = 10$ trials for each $H \in \{10, 20, 30, 40, 50\}$. The results are given in \cref{fig: noviol_approx}. We see the no-violation scheme consistently was close to the approximate value, which indicates it is achieving nearly optimal value. Furthermore, we see both methods scale as expected: roughly quadratically in $H$.

\begin{figure}
    \centering
    \includegraphics[scale=.55]{images/noviol_approx_value_small.png}
    \includegraphics[scale=.55]{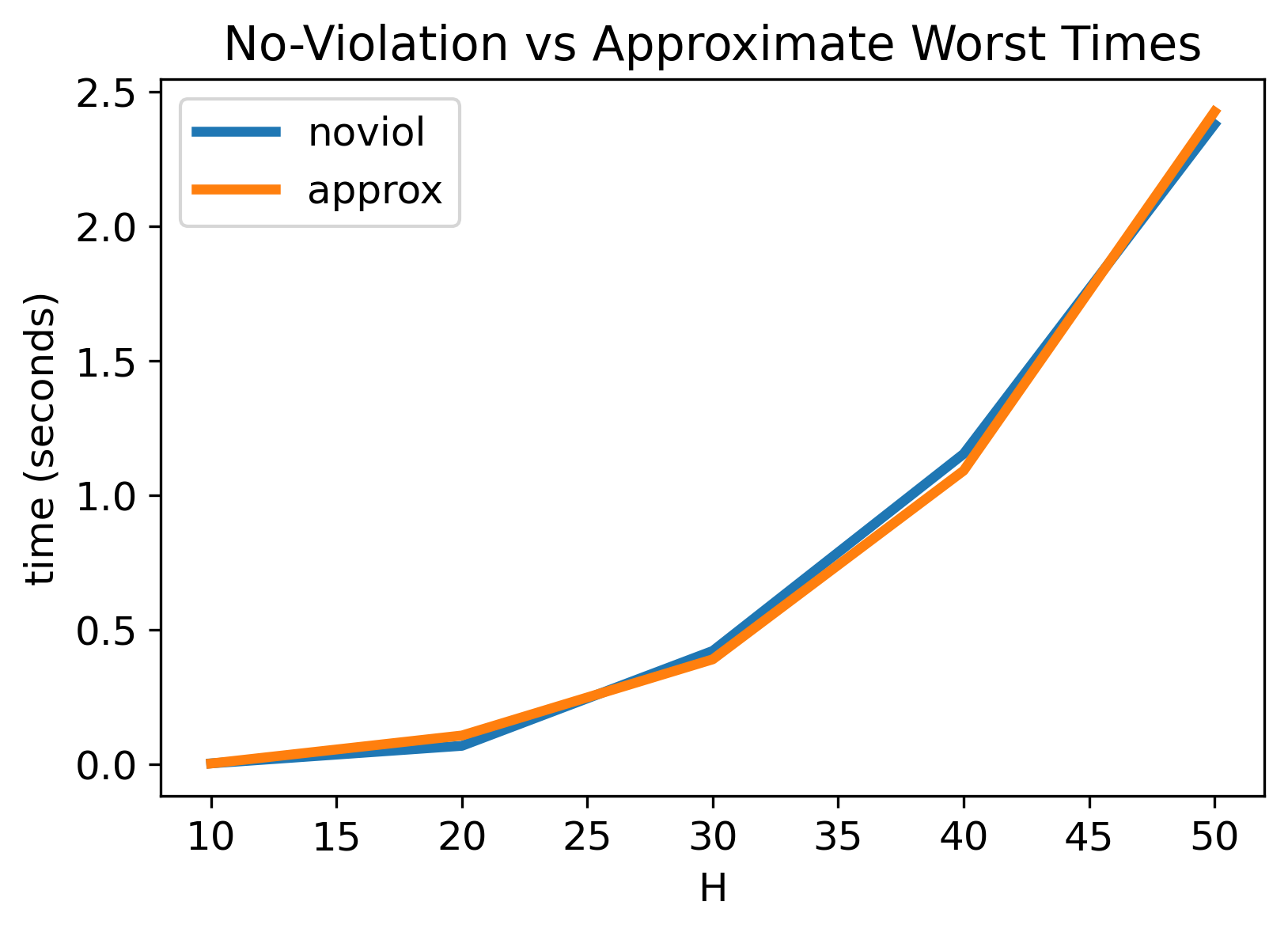}
    \caption{Our No-Violation Scheme vs Our Approximation Scheme on hard cMDP instances.}
    \label{fig: noviol_approx}
\end{figure}

\subsection{Approximation Scale Test}\label{subsec: scale}

Lastly, we wanted to see how large an instance the approximation scheme (and equivalently the no-violation scheme), could handle. We tested the scheme with $H_{\max} = 100$. Specifically, we did $N = 10$ trials for each $H \in \set{10, 20, \ldots, 100}$. This time, we tried both $\epsilon = .1$ and the even larger $\epsilon = 1$. The results are summarized in \cref{fig: approx_scale}. We see for $\epsilon = .01$ the quadratic-like growth is still present, and yet the scheme handled a huge horizon of $100$ in a mere $2$ seconds. For $\epsilon = 1$ the results are even more striking with a maximum run time of $.0008$ seconds and the solutions are guaranteed to violate the budget by no more than a multiple of $2$! The method handles an instance $10$ times larger than the exact method could and yet runs in a fraction of the time.

\begin{figure}
    \centering
    \includegraphics[scale=.55]{images/approx_scale_small.png}
    \includegraphics[scale=.55]{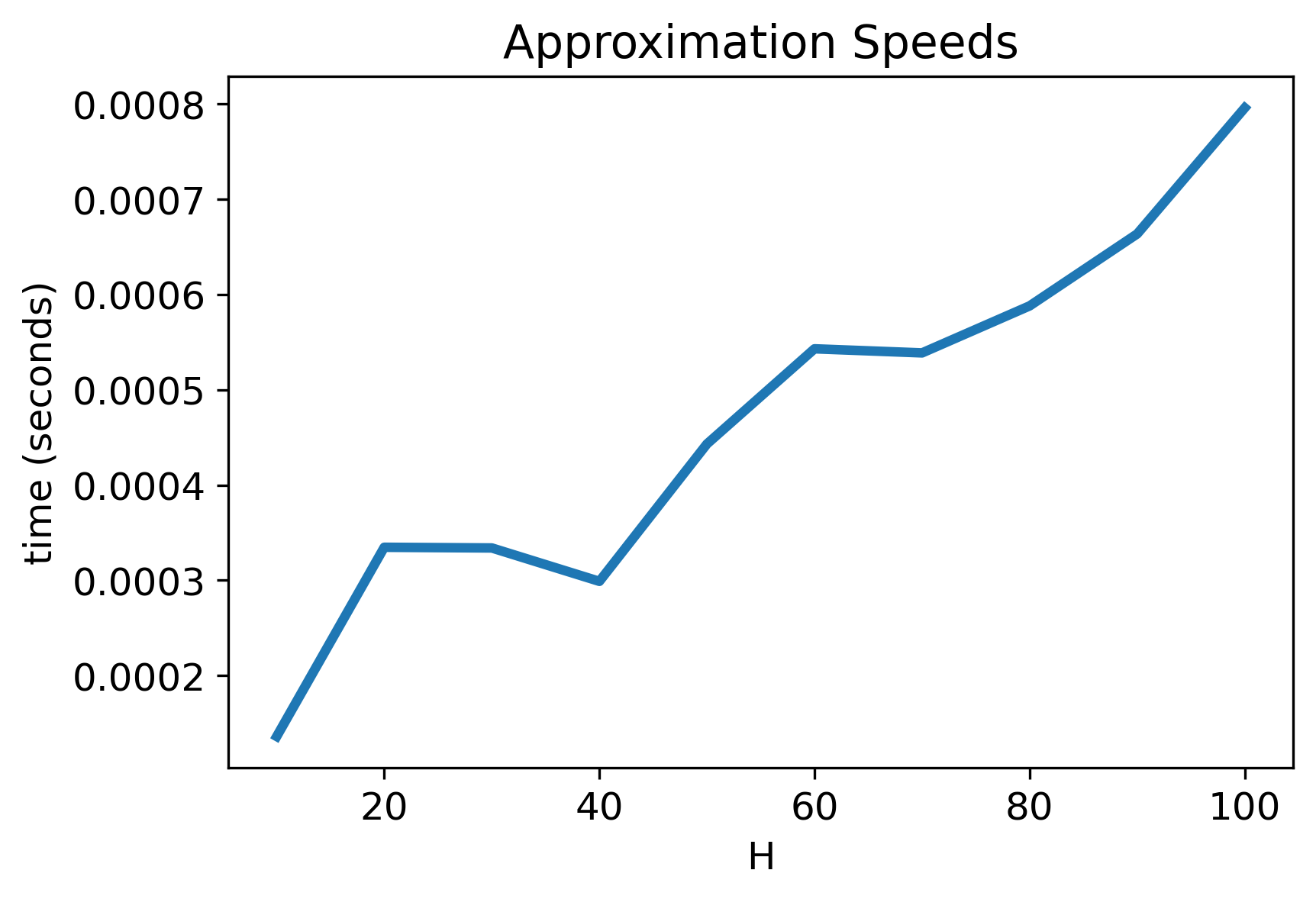}
    \caption{Our Approximation Scheme on large, hard cMDP instances.}
    \label{fig: approx_scale}
\end{figure}

\subsection{Code Details}

We conducted our experiments using standard python3 libraries. We provide our code in a jupyter notebook with an associated database file so that our experiments can be easily reproduced. The notebook already reads in the database by default so no file management is needed. Simply ensure the notebook is in the same directory as the database folder. \footnote{The code can be found at \url{https://github.com/jermcmahan/anytime-constraints.git}}

\end{document}